\documentclass{article}

\usepackage{microtype}
\usepackage{graphicx}
\usepackage{booktabs} %

\usepackage{hyperref}

\usepackage{mathtools}

\DeclarePairedDelimiterX{\KLbase}[2]{(}{)}{%
  #1\;\delimsize\|\;#2%
}
\newcommand{\KL}{D\KLbase}

\usepackage[accepted]{icml2024}

\usepackage{amsmath}
\usepackage{amssymb}
\usepackage{mathtools}
\usepackage{amsthm}

\usepackage[capitalize,noabbrev]{cleveref}

\theoremstyle{plain}
\newtheorem{theorem}{Theorem}[section]

\newtheorem{lemma}[theorem]{Lemma}

\theoremstyle{definition}
\newtheorem{definition}[theorem]{Definition}

\theoremstyle{remark}

\usepackage[textsize=tiny]{todonotes}

\newcommand{\supertitle}{Better \& Faster Large Language Models via Multi-token Prediction}
\icmltitlerunning{\supertitle{}}

\usepackage[utf8]{inputenc} %
\usepackage[T1]{fontenc}    %
\usepackage{hyperref}       %
\usepackage{url}            %
\usepackage{booktabs}       %
\usepackage{amsfonts}       %
\usepackage{nicefrac}       %
\usepackage{microtype}      %
\usepackage{xcolor}         %
\usepackage{multirow}

\usepackage{subcaption}
\usepackage{graphicx}
\usepackage{listings}
\usepackage{amsmath}
\usepackage{amsthm}
\usepackage{float}
\usepackage{enumitem}
\usepackage{bbm}    %
\usepackage{todonotes}

\DeclareMathOperator*{\Var}{Var}
\DeclareMathOperator*{\E}{\mathbb{E}}

\begin{document}

\twocolumn[
\icmltitle{\supertitle{}}

\icmlsetsymbol{equal}{*}
\icmlsetsymbol{adv}{+}

\begin{icmlauthorlist}
\icmlauthor{Fabian Gloeckle}{equal,fair,ponts}
\icmlauthor{Badr Youbi Idrissi}{equal,fair,upsaclay}
\icmlauthor{Baptiste Rozière}{fair}
\icmlauthor{David Lopez-Paz}{adv,fair}
\icmlauthor{Gabriel Synnaeve}{adv,fair}
\end{icmlauthorlist}

\icmlaffiliation{fair}{FAIR at Meta}
\icmlaffiliation{ponts}{CERMICS Ecole des Ponts ParisTech}
\icmlaffiliation{upsaclay}{LISN Université Paris-Saclay}

\icmlcorrespondingauthor{Fabian Gloeckle}{fgloeckle@meta.com}
\icmlcorrespondingauthor{Badr Youbi Idrissi}{byoubi@meta.com}

\icmlkeywords{Machine Learning, ICML}

\vskip 0.3in
]

\printAffiliationsAndNotice{\icmlEqualContribution} %

\begin{abstract}
Large language models such as GPT and Llama are trained with a next-token prediction loss.
In this work, we suggest that training language models to predict \emph{multiple} future tokens at once results in higher sample efficiency.
More specifically, at each position in the training corpus, we ask the model to predict the following $n$ tokens using $n$ independent output heads, operating on top of a shared model trunk. 
Considering multi-token prediction as an auxiliary training task, we measure improved downstream capabilities with no overhead in training time for both code and natural language models.
The method is increasingly useful for larger model sizes, and keeps its appeal when training for multiple epochs. Gains are especially pronounced on \emph{generative} benchmarks like coding, where our models consistently outperform strong baselines by several percentage points. Our 13B parameter models solves 12~\% more problems on HumanEval and 17~\% more on MBPP than comparable next-token models. Experiments on small algorithmic tasks demonstrate that multi-token prediction is favorable for the development of induction heads and algorithmic reasoning capabilities.
As an additional benefit, models trained with 4-token prediction are up to $3\times$ faster at inference, even with large batch sizes.
\end{abstract}

\section{Introduction}
Humanity has condensed its most ingenious undertakings, surprising findings and beautiful productions into text.
Large Language Models (LLMs) trained on all of these corpora are able to extract impressive amounts of world knowledge, as well as basic reasoning capabilities by implementing a simple---yet powerful---unsupervised learning task: next-token prediction.
Despite the recent wave of impressive achievements~\citep{openai2023gpt4}, next-token prediction remains an inefficient way of acquiring language, world knowledge and reasoning capabilities.
More precisely, teacher forcing with next-token prediction latches on local patterns and overlooks ``hard'' decisions. 
Consequently, it remains a fact that state-of-the-art next-token predictors call for orders of magnitude more data than human children to arrive at the same level of fluency~\citep{frank2023bridging}.

\vspace{0.2em}

\begin{figure}[ht!]

    \centering
    \includegraphics[width=\linewidth]{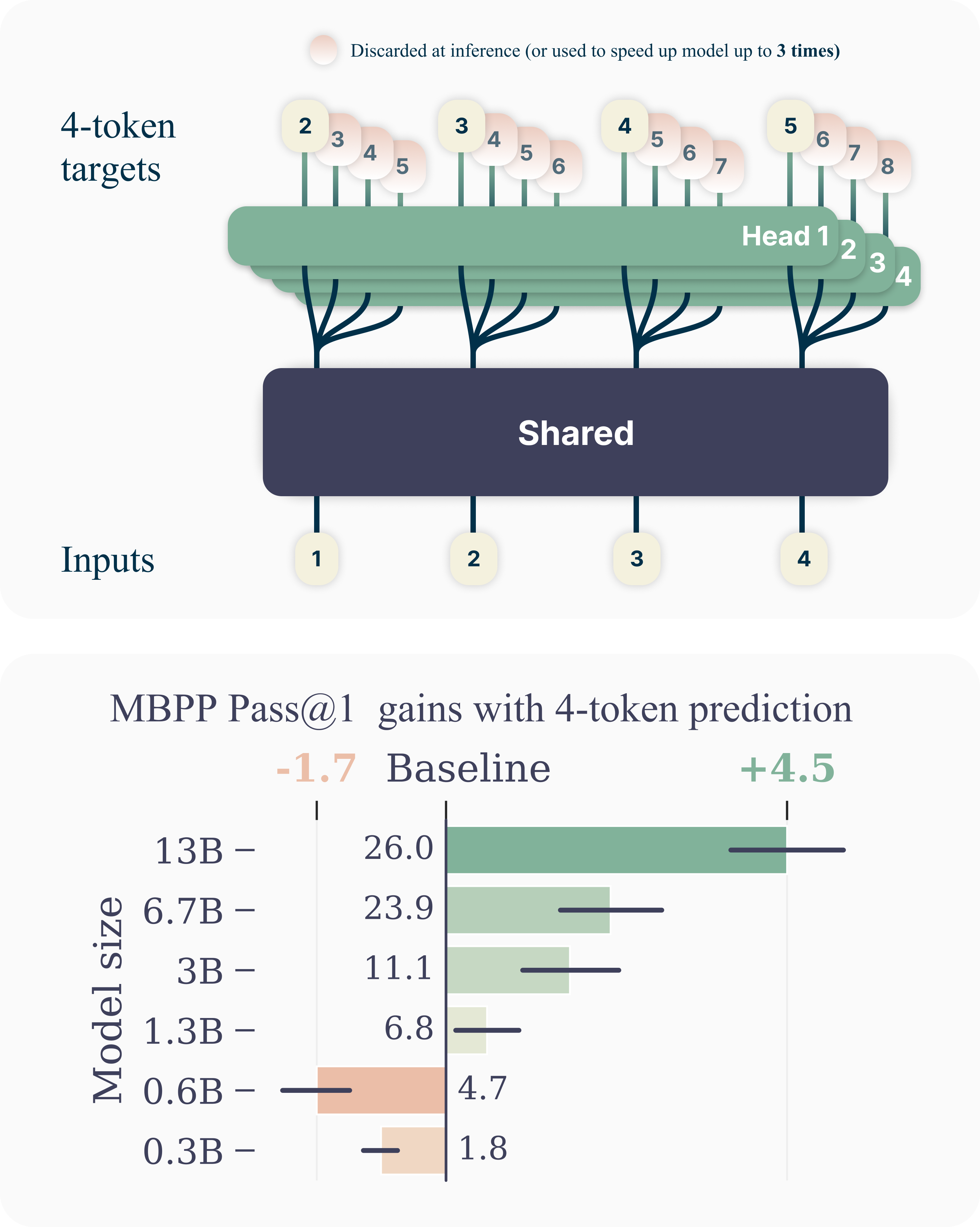}
    \vspace{0.2em}
    \caption{
    \textbf{Overview of multi-token prediction.}
    (Top) During training, the model predicts $4$ future tokens at once, by means of a shared trunk and $4$ dedicated output heads.
    During inference, we employ only the next-token output head.
    Optionally, the other three heads may be used to speed-up inference time.
    (Bottom) Multi-token prediction improves pass@1 on the MBPP code task, significantly so as model size increases.
    Error bars are confidence intervals of 90\% computed with bootstrapping over dataset samples.
    }
    \label{fig:main_figure_col}
\end{figure}

In this study, we argue that training LLMs to \emph{predict multiple tokens} at once will drive these models toward better sample efficiency.
As anticipated in~\cref{fig:main_figure_col}, multi-token prediction instructs the LLM to predict the $n$ future tokens from each position in the training corpora, all at once and in parallel~\citep{qi2020prophetnet}.

\paragraph{Contributions}
While multi-token prediction has been studied in previous literature~\citep{qi2020prophetnet}, the present work offers the following contributions:
\begin{enumerate}
  \item We propose a simple multi-token prediction architecture with no train time or memory overhead (Section~\ref{sect:method}).
  \item We provide experimental evidence that this training paradigm is beneficial at scale, with models up to 13B parameters solving around 15\% more code problems on average (Section~\ref{sect:results}).
  \item Multi-token prediction enables self-speculative decoding, making models up to 3 times faster at inference time across a wide range of batch-sizes (Section~\ref{sect:decoding}).
\end{enumerate}

While cost-free and simple, multi-token prediction is an effective modification to train stronger and faster transformer models.
We hope that our work spurs interest in novel auxiliary losses for LLMs well beyond next-token prediction, as to improve the performance, coherence, and reasoning abilities of these fascinating models.

\section{Method}
\label{sect:method}
Standard language modeling learns about a large text corpus $x_1, \ldots x_T$ by implementing a next-token prediction task.
Formally, the learning objective is to minimize the cross-entropy loss
\begin{align}
     L_1
     &= - \sum_t \log P_\theta(x_{t+1} \mid x_{t:1}),
\end{align}
where $P_\theta$ is our large language model under training, as to maximize the probability of $x_{t+1}$ as the next future token, given the history of past tokens $x_{t:1} = x_{t}, \ldots, x_{1}$.

In this work, we generalize the above by implementing a multi-token prediction task, where at each position of the training corpus, the model is instructed to predict $n$ future tokens at once.
This translates into the cross-entropy loss
\begin{equation}
    L_n = - \sum_t \log P_\theta(x_{t+n:t+1} \mid x_{t:1}).
\end{equation}
To make matters tractable, we assume that our large language model $P_\theta$ employs a shared trunk to produce a latent representation $z_{t:1}$ of the observed context $x_{t:1}$, then fed into $n$ independent heads to predict in parallel each of the $n$ future tokens (see~\cref{fig:main_figure_col}).
This leads to the following factorization of the multi-token prediction cross-entropy loss:
\begin{align*} \label{eq:loss-n}
    L_n
    &= - \sum_t \log P_\theta(x_{t+n:t+1} \mid z_{t:1}) \cdot P_\theta(z_{t:1} \mid x_{t:1}) \\
    &= - \sum_t \sum_{i=1}^{n} \log P_\theta(x_{t+i} \mid z_{t:1}) \cdot P_\theta(z_{t:1} \mid x_{t:1}).
\end{align*}

In practice, our architecture consists of a shared transformer trunk $f_s$ producing the hidden representation $z_{t:1}$ from the observed context $x_{t:1}$, $n$ independent output heads implemented in terms of transformer layers $f_{h_i}$, and a shared unembedding matrix $f_u$.
Therefore, to predict $n$ future tokens, we compute:
$$P_\theta(x_{t+i} \mid x_{t:1}) = \text{softmax}(f_u(f_{h_i}(f_s(x_{t:1})))),$$
for $i=1,\ldots n$, where, in particular, $P_\theta(x_{t+1} \mid x_{t:1})$ is our next-token prediction head.
See Appendix~\ref{app:architecture} for other variations of multi-token prediction architectures.

\paragraph{Memory-efficient implementation}
One big challenge in training multi-token predictors is reducing their GPU memory utilization.
To see why this is the case, recall that in current LLMs the vocabulary size $V$ is much larger than the dimension $d$ of the latent representation---therefore, logit vectors become the GPU memory usage bottleneck.
Naive implementations of multi-token predictors that materialize all logits and their gradients, both of shape $(n, V)$, severely limit the allowable batch-size and average GPU memory utilization.
Because of these reasons, in our architecture we propose to carefully adapt the sequence of forward and backward operations, as illustrated in Figure~\ref{fig:backward}.
In particular, after the forward pass through the shared trunk $f_s$, we sequentially compute the forward \emph{and} backward pass of each independent output head $f_i$, accumulating gradients at the trunk.
While this creates logits (and their gradients) for the output head $f_i$, these are freed before continuing to the next output head $f_{i+1}$, requiring the long-term storage only of the $d$-dimensional trunk gradient $\partial L_n / \partial f_s$.
In sum, we have reduced the peak GPU memory utilization from $O(nV + d)$ to $O(V + d)$, at no expense in  runtime (Table~\ref{tab:wps}).

\paragraph{Inference}
During inference time, the most basic use of the proposed architecture is vanilla next-token autoregressive prediction using the next-token prediction head $P_\theta(x_{t+1} \mid x_{t:1})$, while discarding all others.
However, the additional output heads can be leveraged to speed up decoding from the next-token prediction head with \emph{self-speculative decoding} methods such as blockwise parallel decoding~\citep{stern2018blockwise}---a variant of speculative decoding~\citep{leviathan2023fast} without the need for an additional draft model---and speculative decoding with Medusa-like tree attention~\citep{cai2024medusa}.

\begin{figure}[t!]
    \centering
    \includegraphics[width=\linewidth]{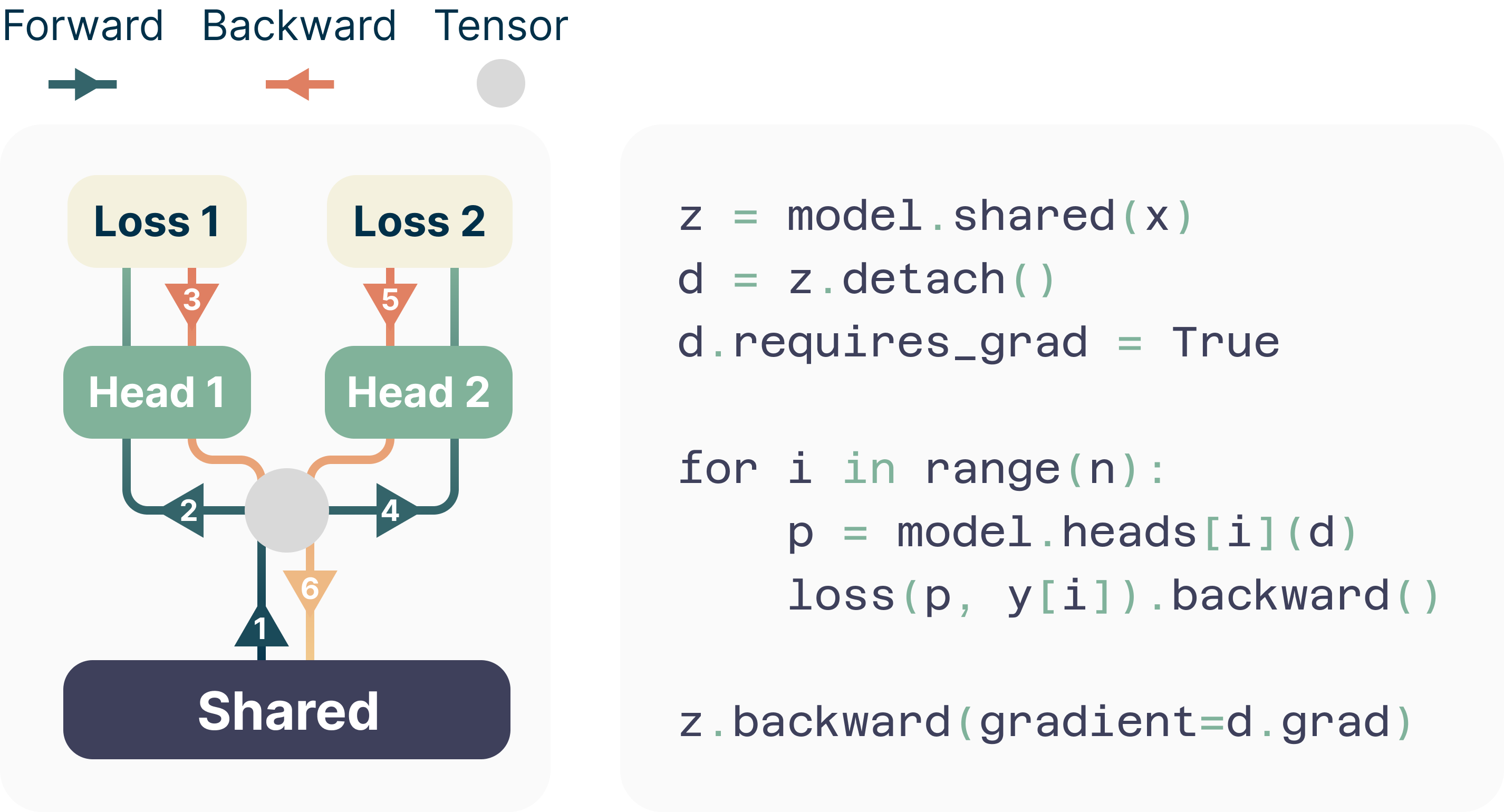}

    \caption{\textbf{Order of the forward/backward in an $n$-token prediction model with $n = 2$ heads.} By performing the forward/backward on the heads in sequential order, we avoid materializing all unembedding layer gradients in memory simultaneously and reduce peak GPU memory usage.}
    \label{fig:backward}
\end{figure}

\section{Experiments on real data}
\label{sect:results}
\begin{table*}[t]
    \centering
    \caption{\textbf{Multi-token prediction improves performance and unlocks efficient byte level training.} We compare models with 7B parameters trained from scratch on 200B and on 314B bytes of code on the MBPP \citep{austin2021program}, HumanEval \citep{chen2021evaluating} and APPS \citep{hendrycks2021measuring} benchmarks. Multi-token prediction largely outperforms next token prediction on these settings. All numbers were calculated using the estimator from \citet{chen2021evaluating} based on 200 samples per problem. The temperatures were chosen optimally (based on test scores; i.e. these are oracle temperatures) for each model, dataset and pass@k and are reported in Table~\ref{tab:optimal_temps}.}
    \label{tab:future_tokens_results}
\begin{tabular}{cccccccccccc}
\toprule
 \multirow[c]{2}{*}{Training data} & \multirow[c]{2}{*}{Vocabulary} & \multirow[c]{2}{*}{n} & \multicolumn{3}{c}{MBPP} & \multicolumn{3}{c}{HumanEval} & \multicolumn{3}{c}{APPS/Intro} \\
 \cmidrule(l){4-6}\cmidrule(l){7-9} \cmidrule(l){10-12}
 &  &  & @1 & @10 & @100 & @1 & @10 & @100 & @1 & @10 & @100 \\
\midrule
\multirow[c]{4}{*}{\shortstack{313B bytes\\(0.5 epochs)}} & \multirow[c]{4}{*}{bytes} & 1 & 19.3 & 42.4 & 64.7 & 18.1 & 28.2 & 47.8 & 0.1 & 0.5 & 2.4 \\
 &  & 8 & \bfseries 32.3 & \bfseries 50.0 & \bfseries 69.6 & \bfseries 21.8 & \bfseries 34.1 & \bfseries 57.9 & \bfseries 1.2 & \bfseries 5.7 & \bfseries 14.0 \\
 &  & 16 & 28.6 & 47.1 & 68.0 & 20.4 & 32.7 & 54.3 & 1.0 & 5.0 & 12.9 \\
 &  & 32 & 23.0 & 40.7 & 60.3 & 17.2 & 30.2 & 49.7 & 0.6 & 2.8 & 8.8 \\
 \midrule
\multirow[c]{5}{*}{\shortstack{200B tokens\\(0.8 epochs)}} & \multirow[c]{5}{*}{32k tokens} & 1 & 30.0 & 53.8 & 73.7 & 22.8 & 36.4 & 62.0 & 2.8 & 7.8 & 17.4 \\
 &  & 2 & 30.3 & 55.1 & 76.2 & 22.2 & 38.5 & 62.6 & 2.1 & 9.0 & 21.7 \\
 &  & 4 & \bfseries 33.8 & \bfseries 55.9 & \bfseries 76.9 & \bfseries 24.0 & \bfseries 40.1 & \bfseries 66.1 & 1.6 & 7.1 & 19.9 \\
 &  & 6 & 31.9 & 53.9 & 73.1 & 20.6 & 38.4 & 63.9 & \bfseries 3.5 & \bfseries 10.8 & \bfseries 22.7 \\
 &  & 8 & 30.7 & 52.2 & 73.4 & 20.0 & 36.6 & 59.6 & 3.5 & 10.4 & 22.1 \\
 \midrule
\multirow[c]{2}{*}{\shortstack{1T tokens\\(4 epochs)}} & \multirow[c]{2}{*}{32k tokens} & 1 & 40.7 & 65.4 & 83.4 & 31.7 &  57.6 & 83.0 & \bfseries 5.4 & \bfseries 17.8 & \bfseries 34.1 \\
 &  & 4 & \bfseries 43.1 & 65.9 &  83.7 & 31.6 & 57.3 & \bfseries 86.2 & 4.3 & 15.6 & 33.7 \\
\bottomrule
\end{tabular}
\end{table*}

We demonstrate the efficacy of multi-token prediction losses by seven large-scale experiments.
Section~\ref{sect:model-scaling} shows how multi-token prediction is increasingly useful when growing the model size.
Section~\ref{sect:decoding} shows how the additional prediction heads can speed up inference by a factor of \textbf{$3\times$} using speculative decoding. 
Section~\ref{sect:byte-level} demonstrates how multi-token prediction promotes learning longer-term patterns, a fact most apparent in the extreme case of byte-level tokenization.
Section~\ref{sect:optimal-future} shows that $4$-token predictor leads to strong gains with a tokenizer of size $32$k.
Section~\ref{sect:multi-epoch} illustrates that the benefits of multi-token prediction remain for training runs with multiple epochs. 
Section~\ref{sect:finetuning} showcases the rich representations promoted by pretraining with multi-token prediction losses by finetuning on the CodeContests dataset~\citep{li2022competition}.
Section~\ref{sect:nlp} shows that the benefits of multi-token prediction carry to natural language models, improving \emph{generative} evaluations such as summarization, while not regressing significantly on standard benchmarks based on multiple choice questions and negative log-likelihoods. 

To allow fair comparisons between next-token predictors and $n$-token predictors, the experiments that follow always compare models with an equal amount of parameters.
That is, when we add $n-1$ layers in future prediction heads, we remove $n-1$ layers from the shared model trunk.
Please refer to Table~\ref{tab:models} for the model architectures and to Table~\ref{tab:hyperparams} for an overview of the hyperparameters we use in our experiments.

\begin{figure}[ht!]

    \centering
    \includegraphics[width=\linewidth]{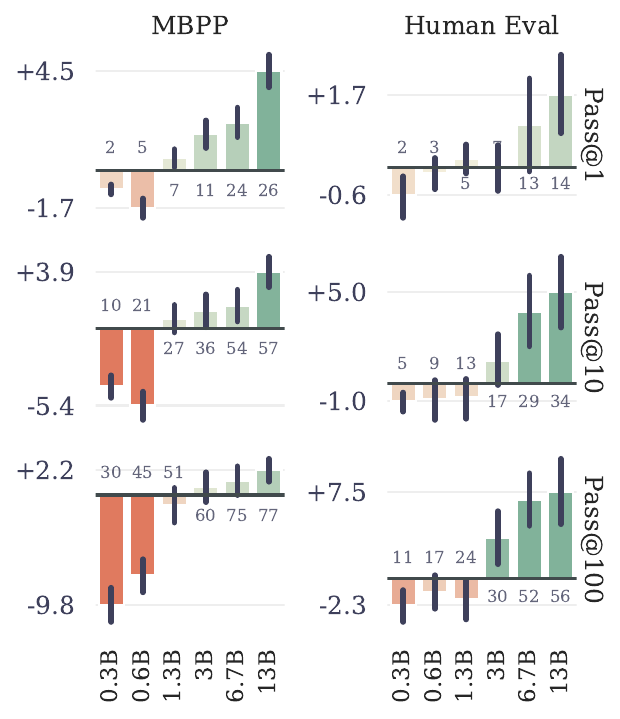}
    \caption{\textbf{Results of $n$-token prediction models on MBPP by model size.} We train models of six sizes in the range or 300M to 13B total parameters on code,
    and evaluate pass@1,10,100 on the MBPP \citep{austin2021program} and HumanEval \cite{chen2021evaluating} benchmark  
    with 1000 samples.
    Multi-token prediction models are worse than the baseline for small model sizes, but outperform the baseline at scale. 
    Error bars are confidence intervals of 90\% computed with bootstrapping over dataset samples. }
    \label{fig:scaling_model}
\end{figure}

\subsection{Benefits scale with model size}
\label{sect:model-scaling}
To study this phenomenon, we train models of six sizes in the range 300M to 13B parameters from scratch on at least 91B tokens of code.
The evaluation results in Figure~\ref{fig:scaling_model} for MBPP \citep{austin2021program} and HumanEval \citep{chen2021evaluating} show that it is possible, with the exact same computational budget, to squeeze much more performance out of large language models given a fixed dataset using multi-token prediction. 

We believe this \emph{usefulness only at scale} to be a likely reason why multi-token prediction has so far been largely overlooked as a promising training loss for large language model training. 

\subsection{Faster inference}
\label{sect:decoding}

We implement greedy \emph{self-speculative decoding} \cite{stern2018blockwise} with heterogeneous batch sizes using xFormers \citep{xFormers2022} and measure decoding speeds of our best 4-token prediction model with 7B parameters on completing prompts taken from a test dataset of code and natural language (Table~\ref{tab:decoding}) not seen during training. We observe a speedup of $\mathbf{3.0\times}$ on code with an average of 2.5 accepted tokens out of 3 suggestions on code, and of $2.7\times$ on text. On an 8-byte prediction model, the inference speedup is $6.4\times$ (Table~\ref{tab:decoding_bytes}). Pretraining with multi-token prediction allows the additional heads to be much more accurate than a simple finetuning of a next-token prediction model, thus allowing our models to unlock self-speculative decoding's full potential. %

\subsection{Learning global patterns with multi-byte prediction}
\label{sect:byte-level}
To show that the next-token prediction task latches to local patterns, we went to the extreme case of byte-level tokenization by training a 7B parameter byte-level transformer  on 314B bytes, which is equivalent to around 116B tokens. The 8-byte prediction model achieves astounding improvements compared to next-byte prediction, solving 67\% more problems on MBPP pass@1 and 20\% more problems on HumanEval pass@1. 

Multi-byte prediction is therefore a very promising avenue to unlock efficient training of byte-level models. Self-speculative decoding can achieve speedups of 6 times for the 8-byte prediction model, which would allow to fully compensate the cost of longer byte-level sequences at inference time and even be faster than a next-token prediction model by nearly two times. The 8-byte prediction model is a strong byte-based model, approaching the performance of token-based models despite having been trained on $1.7\times$ less data.

\subsection{Searching for the optimal $n$}
\label{sect:optimal-future}

To better understand the effect of the number of predicted tokens, we did comprehensive ablations on models of scale 7B trained on 200B tokens of code. We try $n=1,2,4,6$ and $8$ in this setting. Results in table \ref{tab:future_tokens_results} show that training with 4-future tokens outperforms all the other models consistently throughout HumanEval and MBPP for pass at 1, 10 and 100 metrics: +3.8\%, +2.1\% and +3.2\% for MBPP and +1.2\%, +3.7\% and +4.1\% for HumanEval. Interestingly, for APPS/Intro, $n=6$ takes the lead with +0.7\%, +3.0\% and +5.3\%. It is very likely that the optimal window size depends on input data distribution. As for the byte level models the optimal window size is more consistent (8 bytes) across these benchmarks.

\subsection{Training for multiple epochs}
\label{sect:multi-epoch}
Multi-token training still maintains an edge on next-token prediction when trained on multiple epochs of the same data. The improvements diminish but we still have a +2.4\% increase on pass@1 on MBPP and +3.2\% increase on pass@100 on HumanEval, while having similar performance for the rest. As for APPS/Intro, a window size of 4 was already not optimal with 200B tokens of training.

\subsection{Finetuning multi-token predictors}
\label{sect:finetuning}
Pretrained models with multi-token prediction loss also outperform next-token models for use in finetunings. We evaluate this by finetuning 7B parameter models from Section~\ref{sect:byte-level} on the CodeContests dataset \citep{li2022competition}. We compare the 4-token prediction model with the next-token prediction baseline, and include a setting where the 4-token prediction model is stripped off its additional prediction heads and finetuned using the classical next-token prediction target. According to the results in Figure~\ref{fig:dm_contests}, both ways of finetuning the 4-token prediction model outperform the next-token prediction model on pass@k across $k$. This means the models are both better at understanding and solving the task and at generating diverse answers. Note that CodeContests is the most challenging coding benchmark we evaluate in this study. Next-token prediction finetuning on top of 4-token prediction pretraining appears to be the best method overall, in line with the classical paradigm of pretraining with auxiliary tasks followed by task-specific finetuning. Please refer to Appendix~\ref{app:finetuning} for details.
\begin{figure}[ht!]

        \centering
        \includegraphics[width=0.8\linewidth]{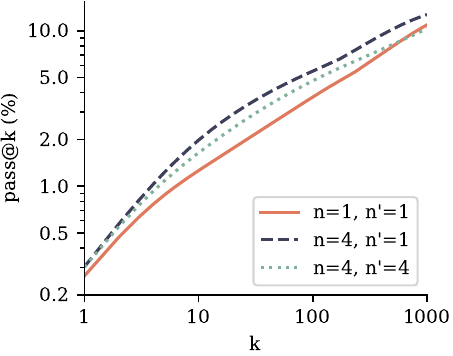}
    
    \caption{\textbf{Comparison of finetuning performance on CodeContests.}
    We finetune a $4$-token prediction model on CodeContests \citep{li2022competition} (train split) using $n'$-token prediction as training loss with $n' = 4$ or $n' = 1$, and compare to a finetuning of the next-token prediction baseline model ($n = n' = 1$). For evaluation, we generate 1000 samples per test problem for each temperature $T \in \{0.5, 0.6, 0.7, 0.8, 0.9\}$, and compute pass@k for each value of $k$ and $T$. Shown is $k \mapsto \max_T \mathrm{pass\_at}(k, T)$, i.e. we grant access to a temperature oracle. We observe that both ways of finetuning the 4-token prediction model outperform the next-token prediction baseline. Intriguingly, using next-token prediction finetuning on top of the 4-token prediction model appears to be the best method overall.
}
    \label{fig:dm_contests}
\end{figure}

\vspace{-1em}
\subsection{Multi-token prediction on natural language}
\label{sect:nlp}

To evaluate multi-token prediction training on natural language, we train models of size 7B parameters on 200B tokens of natural language with a 4-token, 2-token and next-token prediction loss, respectively. In Figure ~\ref{fig:nlp_evol}, we evaluate the resulting checkpoints on 6 standard NLP benchmarks. On these benchmarks, the 2-future token prediction model performs on par with the next-token prediction baseline throughout training. The 4-future token prediction model suffers a performance degradation. Detailed numbers are reported in Appendix~\ref{app:choice_nlp}.

\begin{figure}[t!]

    \centering
    \includegraphics[width=0.8\columnwidth]{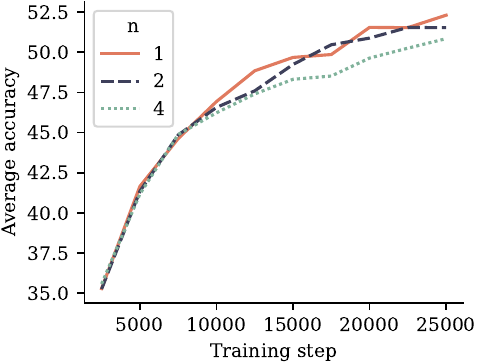}
    \caption{\textbf{Multi-token training with 7B models doesn't improve performance on choice tasks.} This figure shows the evolution of average accuracy of 6 standard NLP benchmarks. Detailed results in Appendix~\ref{app:choice_nlp} for 7B models trained on 200B tokens of language data. The 2 future token model has the same performance as the baseline and the 4 future token model regresses a bit. Larger model sizes might be necessary to see improvements on these tasks.
    }
    \label{fig:nlp_evol}
\end{figure}

However, we do not believe that multiple-choice and likelihood-based benchmarks are suited to effectively discern \emph{generative capabilities} of language models. In order to avoid the need for human annotations of generation quality or language model judges---which comes with its own pitfalls, as pointed out by \citet{koo2023benchmarking}---we conduct evaluations on summarization and natural language mathematics benchmarks and compare pretrained models with training sets sizes of 200B and 500B tokens and with next-token and multi-token prediction losses, respectively.

For summarization, we use eight benchmarks where ROUGE metrics \citep{lin-2004-rouge} with respect to a ground-truth summary allow automatic evaluation of generated texts. We finetune each pretrained model on each benchmark's training dataset for three epochs and select the checkpoint with the highest ROUGE-L $F_1$ score on the validation dataset.
Figure~\ref{fig:summarization_scale} shows that multi-token prediction models  with both $n=2$ and $n=4$ improve over the next-token baseline in ROUGE-L $F_1$ scores for both training dataset sizes,
with the performance gap shrinking with larger dataset size. All metrics can be found in Appendix~\ref{app:summarzation_all}.
\begin{figure}[t!]
    \centering
    \includegraphics[width=0.8\linewidth]{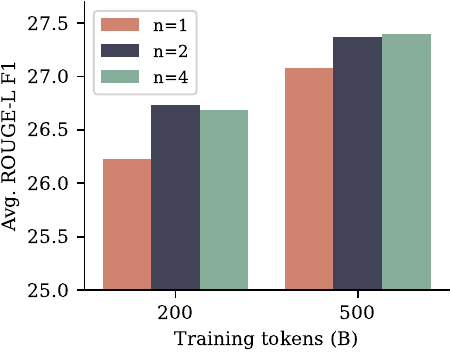}
    \caption{\textbf{Performance on abstractive text summarization.}
    Average ROUGE-L (longest common subsequence overlap) $F_1$ score for 7B models trained on 200B and 500B tokens of natural language on eight summarization benchmarks.
    We finetune the respective models on each task's training data separately for three epochs and select the checkpoints with highest ROUGE-L $F_1$ validation score.
     Both $n=2$ and $n=4$ multi-token prediction models have an advantage over next-token prediction models.
    Individual scores per dataset and more details can be found in Appendix~\ref{app:summarzation_all}.
    }
    \label{fig:summarization_scale}
\end{figure}

For natural language mathematics, we evaluate the pretrained models in 8-shot mode on the GSM8K benchmark \citep{cobbe2021training} and measure accuracy of the final answer produced after a chain-of-thought elicited by the fewshot examples. We evaluate pass@k metrics to quantify diversity and correctness of answers like in code evaluations and use sampling temperatures between 0.2 and 1.4. The results are depicted in Figure~\ref{fig:gsm8k} in Appendix~\ref{app:gsm8k}. For 200B training tokens, the $n=2$ model clearly outperforms the next-token prediction baseline, while the pattern reverses after 500B tokens and $n=4$ is worse throughout.

\section{Ablations on synthetic data}
What drives the improvements in downstream performance of multi-token prediction models on all of the tasks we have considered? By conducting toy experiments on controlled training datasets and evaluation tasks, we demonstrate that multi-token prediction leads to \emph{qualitative changes in model capabilities and generalization behaviors}.
In particular, Section~\ref{sect:induction} shows that for small model sizes, \emph{induction capability}---as discussed by \citet{olsson2022context}---either only forms when using multi-token prediction as training loss, or it is vastly improved by it. Moreover, Section~\ref{sect:algorithmic} shows that multi-token prediction improves generalization on an arithmetic task, even more so than tripling model size.

\subsection{Induction capability}
\label{sect:induction}
\begin{figure}[t!]
    \centering
    \includegraphics[width=0.8\linewidth]{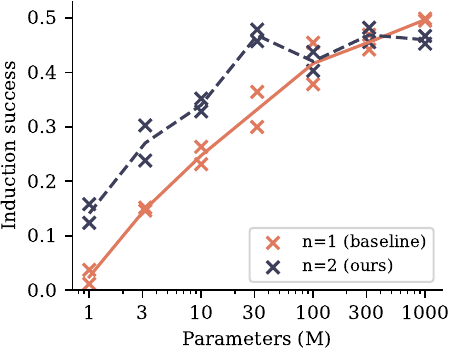}
    \caption{\textbf{Induction capability of $n$-token prediction models.} 
     Shown is accuracy on the second token of two token names that have already been mentioned previously. Shown are numbers for models trained with a next-token and a 2-token prediction loss, respectively, with two independent runs each. The lines denote per-loss averages. For small model sizes, next-token prediction models learn practically no or significantly worse induction capability than 2-token prediction models, with their disadvantage disappearing at the size of 100M nonembedding parameters.
}

    \label{fig:induction_size}
\end{figure}

Induction describes a simple pattern of reasoning that completes partial patterns by their most recent continuation \citep{olsson2022context}. In other words, if a sentence contains ``AB'' and later mentions ``A'', induction is the prediction that the continuation is ``B''. We design a setup to measure induction capability in a controlled way. Training small models of sizes 1M to 1B nonembedding parameters on a dataset of children stories, we measure induction capability by means of an adapted test set: in 100 stories from the original test split, we replace the character names by randomly generated names that consist of two tokens with the tokenizer we employ.
Predicting the first of these two tokens is linked to the semantics of the preceding text, while predicting the second token of each name's occurrence after it has been mentioned at least once can be seen as a pure induction task. In our experiments, we train for up to 90 epochs and perform early stopping with respect to the test metric (i.e. we allow an epoch oracle). Figure~\ref{fig:induction_size} reports induction capability as measured by accuracy on the names' second tokens in relation to model size for two runs with different seeds.

We find that 2-token prediction loss leads to a vastly improved formation of induction capability for models of size 30M nonembedding parameters and below, with their advantage disappearing for sizes of 100M nonembedding parameters and above.\footnote{Note that a perfect score is not reachable in this benchmark as some of the tokens in the names in the evaluation dataset never appear in the training data, and in our architecture, embedding and unembedding parameters are not linked.}
We interpret this finding as follows: multi-token prediction losses help models to learn transferring information across sequence positions, which lends itself to the formation of induction heads and other in-context learning mechanisms. However, once induction capability has been formed, these \emph{learned features} transform induction into a task that can be solved \emph{locally} at the current token and learned with next-token prediction alone. From this point on, multi-token prediction actually hurts on this restricted benchmark---but we surmise that there are higher forms of in-context reasoning to which it further contributes, as evidenced by the results in Section~\ref{sect:model-scaling}. In Figure~\ref{fig:induction_b3g}, we provide evidence for this explanation: replacing the children stories dataset by a higher-quality 9:1 mix of a books dataset with the children stories, we enforce the formation of induction capability early in training by means of the dataset alone. By consequence, except for the two smallest model sizes, the advantage of multi-token prediction on the task disappears: feature learning  of induction features has converted the task into a pure next-token prediction task.

\subsection{Algorithmic reasoning}
\label{sect:algorithmic}
\begin{figure}[t!]
    \centering
    \includegraphics[width=0.95\linewidth]{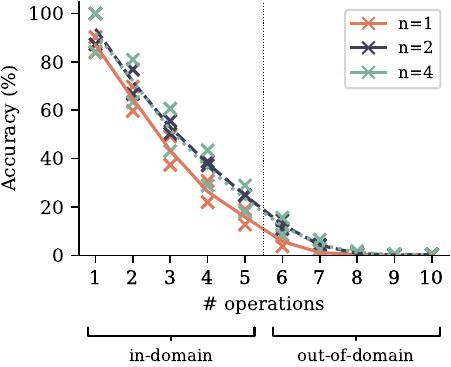}
    \caption{\textbf{Accuracy on a polynomial arithmetic task with varying number of operations per expression.} Training with multi-token prediction losses increases accuracy across task difficulties. In particular, it also significantly improves out-of-domain generalization performance, albeit at a low absolute level. Tripling the model size, on the other hand, has a considerably smaller effect than replacing next-token prediction with multi-token prediction loss (Figure~\ref{fig:poly_scale}). Shown are two independent runs per configuration with 100M parameter models.
}
    \label{fig:poly_0}
\end{figure}

Algorithmic reasoning tasks allow to measure more involved forms of in-context reasoning than induction alone. We train and evaluate models on a task on polynomial arithmetic in the ring $\mathbb{F}_7[X]/(X^5)$ with unary negation, addition, multiplication and composition of polynomials as operations. The coefficients of the operands and the operators are sampled uniformly. %
The task is to return the coefficients of the polynomials corresponding to the resulting expressions. The number $m$ of operations contained in the expressions is selected uniformly from the range from 1 to 5 at training time, and can be used to adjust the difficulty of both in-domain ($m \leq 5$) and out-of-domain ($m > 5$) generalization evaluations. The evaluations are conducted with greedy sampling on a fixed test set of 2000 samples per number of operations. We train models of two small sizes with 30M and 100M nonembedding parameters, respectively. This simulates the conditions of large language models trained on massive text corpora which are likewise under-parameterized and unable to memorize their entire training datasets.

Multi-token prediction improves algorithmic reasoning capabilities as measured by this task across task difficulties (Figure~\ref{fig:poly_0}). In particular, it leads to impressive gains in out-of-distribution generalization, despite the low absolute numbers. Increasing the model size from 30M to 100M parameters, on the other hand, does not improve evaluation accuracy as much as replacing next-token prediction by multi-token prediction does (Figure~\ref{fig:poly_scale}). In Appendix~\ref{app:poly}, we furthermore show that multi-token prediction models retain their advantage over next-token prediction models on this task when trained and evaluated with \emph{pause tokens} \citep{goyal2023think}.

\section{Why does it work? Some speculation}

\label{sec:discussion}

Why does multi-token prediction afford superior performance on coding evaluation benchmarks, and on small algorithmic reasoning tasks?
Our intuition, developed in this section, is that multi-token prediction mitigates the distributional discrepancy between training-time teacher forcing and inference-time autoregressive generation.
We support this view with an illustrative argument on the \emph{implicit weights} multi-token prediction assigns to tokens depending on their relevance for the continuation of the text, as well as with an information-theoretic decomposition of multi-token prediction loss.

\subsection{Lookahead reinforces choice points}
\label{sect:derailing}
\begin{figure}[t!]
    \centering
    \includegraphics[width=0.95\linewidth]{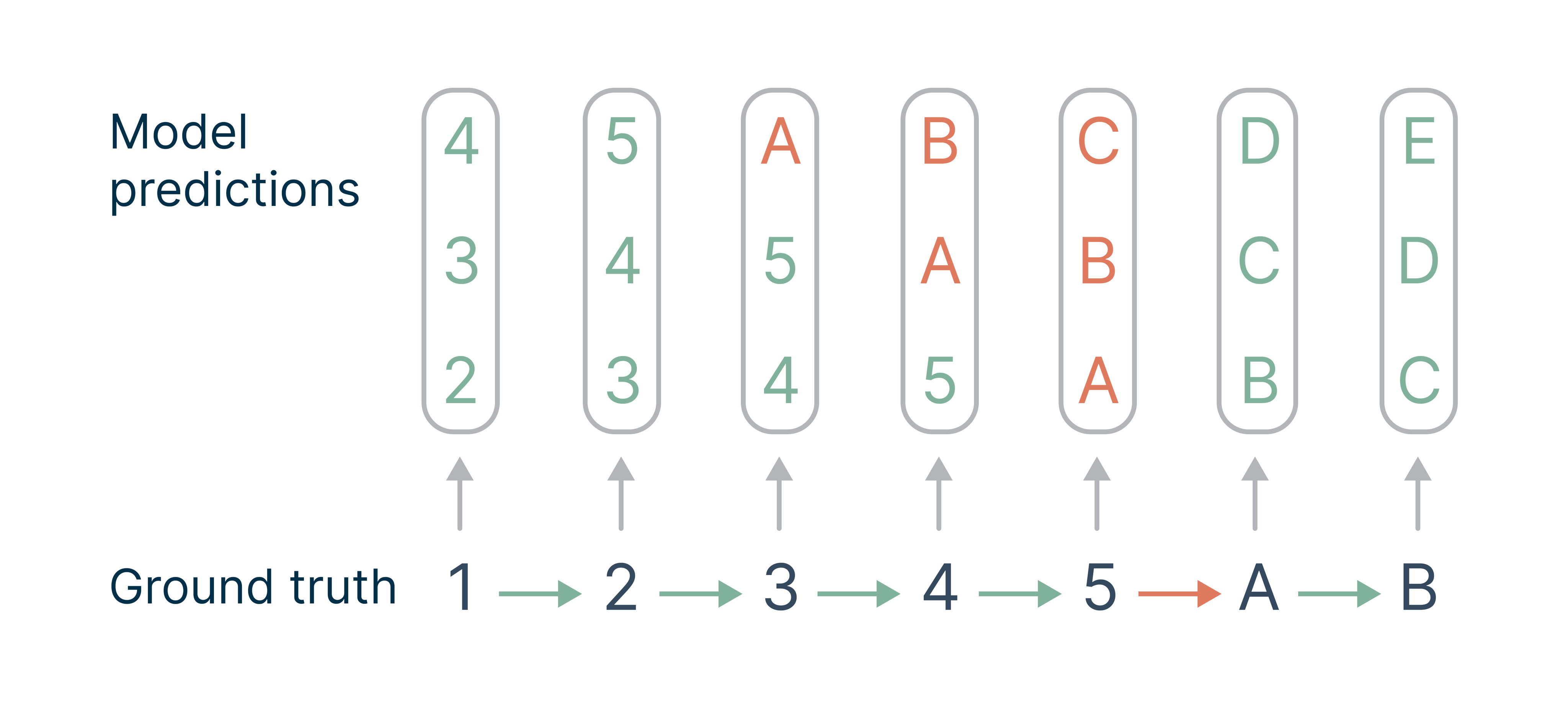}
    \caption{\textbf{Multi-token prediction loss assigns higher implicit weights to \emph{consequential} tokens.} Shown is a sequence in which all transitions except ``5 $\to$ A'' are easy to predict, alongside the corresponding prediction targets in 3-token prediction. Since the consequences of the difficult transition ``5 $\to$ A'' are likewise hard to predict, this transition receives a higher implicit weight in the overall loss via its correlates ``3 $\to$ A'', ..., ``5 $\to$ C''.
}

\label{fig:weights}
\end{figure}

Not all token decisions are equally important for generating useful texts from language models \citep{bachmann2024pitfalls,lin2024rho1}.
While some tokens allow stylistic variations that do not constrain the remainder of the text, others represent \emph{choice points} that are linked with higher-level semantic properties of the text and may decide whether an answer is perceived as useful or \emph{derailing}.

Multi-token prediction implicitly assigns weights to training tokens depending on how closely they are correlated with their successors.
As an illustrative example, consider the sequence depicted in Figure~\ref{fig:weights} where one transition is a hard-to-predict choice point while the other transitions are considered ``inconsequential''. Inconsequential transitions following a choice point are likewise hard to predict in advance. By marking and counting loss terms, we find that $n$-token prediction associates a weight of $\frac{n(n+1)}{2}$ to choice points via their correlates, and a smaller weight of $n$ to inconsequential points.
Please refer to Appendix~\ref{app:choice} for more details.
Generally, we believe that the quality of text generations depends on picking the right decisions at choice points, and that $n$-token prediction losses promote those.

\subsection{Information-theoretic argument}
\label{sect:mismatch}

Language models are typically trained by teacher-forcing, where the model receives the ground truth for each future token during training. 
However, during test time generation is unguided and autoregressive, whereby errors accumulate.
Teacher-forcing, we argue, encourages models to focus on predicting well in the very short term, at the potential expense of ignoring longer-term dependencies in the overall structure of the generated sequence.

To illustrate the impact of multi-token prediction, consider the following information-theoretic argument.
Here, $X$ denotes the next future token, and $Y$ the second-next future token.
The production of both of these tokens is conditioned on some observed, input context $C$, that we omit from our equations for simplicity.
When placed before token $X$, vanilla next-token prediction concerns the quantity $H(X)$, while multi-token prediction with $n=2$ aims at $H(X) + H(Y)$.
We decompose these two quantities as:
\begin{align*}
  H(X)        &= H(X \mid Y) +   I(X; Y),\\
  H(X) + H(Y) &= H(X \mid Y) + 2 I(X; Y) + H(Y \mid X).
\end{align*}
By discarding the term $H(Y \mid X)$---which appears again when predicting at the following position---we observe that 2-token prediction increases the importance of $I(X; Y)$ by a factor of $2$.
So, multi-token predictors are more accurate at predicting tokens $X$ that are of relevance for the remainder of the text to come. In Appendix~\ref{app:loss-dec}, we give a relative version of the above equations that shows the increased weight of \emph{relative mutual information} in a loss decomposition of 2-token prediction loss.

\section{Related work}

\paragraph{Language modeling losses}

\citet{dong2019unified} and \citet{tay2022ul2} train on a mixture of denoising tasks with different attention masks (full, causal and prefix attention) to bridge the performance gap with next token pretraining on generative tasks. \citet{tay2022ul2} uses the span corruption objective, which replaces spans of tokens with special tokens for the encoder and the decoder then predicts the contents of those spans. Unlike UniLM, this allows full causal training with teacher forcing. Similarly, \citet{yang2019xlnet} train on permuted sequences, while conserving the original positional embeddings, effectively training the model to predict various parts of the sequence given a mix of past and future information. This permuted language modeling is the closest task to ours since it allows predicting beyond the next token. However all of these language modeling tasks train on a small percentage of the input text: on average only 15\% of the tokens are backwarded through. For \citet{dong2019unified}, where the masking is done in BERT style, it is hard to mask more than 15\% since it destroys too much information. For \citet{tay2022ul2}, it is technically possible to have a larger proportion but in practice, the settings used have between 15\% and 25\% of masked tokens. \cite{yang2019xlnet} also makes it possible to train on the whole sequence since it is only permuted, and no information is lost. Yet, in practice, since the completely random permutation is very hard to reconstruct, only 15\% are predicted for training stability reasons. 

\paragraph{Multi-token prediction in language modelling}
\citet{qi2020prophetnet} argue that multi-token prediction encourages planning, improves representations and prevents the overfitting on local patterns that can result from teacher-forced training. However, their technical approach replicates the residual stream $n$-fold while ours allows for compute-matched comparisons and makes the residual representations participate more directly in the auxiliary loss terms. \citet{stern2018blockwise} and \citet{cai2024medusa} propose model finetunings with multi-token prediction for faster inference but do not study the effects of such a loss during pretraining. \citet{pal2023future} use probing methods to show that next-token prediction models are able to predict additional consecutive tokens to a certain extent, but less so than our models which are specifically trained for this task.
\citet{jianyu_leon} observe improvements in language modelling tasks with multi-label binary classification over the occurrence of vocabulary words in the future as an auxiliary learning task.

\paragraph{Self-speculative decoding}
\citet{stern2018blockwise} are, to the best of our knowledge, the first to suggest a speculative decoding scheme for faster inference. Our architecture replaces their linear prediction heads by transformer layers, but is otherwise similar. By reorganizing the order of the forward/backward, we can use all loss terms instead of stochastically picking one head for loss computation. \citet{cai2024medusa} present a more elaborate self-speculative decoding scheme that uses the top-$k$ predictions of each head instead of the best one only. It can be used with the multi-token prediction models we train.

\paragraph{Multi-target prediction}
Multi-task learning is the paradigm of training neural networks jointly on several tasks to improve performance on the tasks of interest \citep{caruana1997multitask}. Learning with such auxiliary tasks allows models to exploit dependencies between target variables and can even be preferable in the case of independent targets \citep{waegeman2019multi}. While more specifically tailored architectures for multi-target prediction are conceivable \citep{spyromitros2016multi,read2021classifier}, modern deep learning approaches usually rely on large shared model trunks with separate prediction heads for the respective tasks \citep{caruana1997multitask,silver2016mastering,lample2022hypertree} like we do. Multi-target prediction has been shown to be a successful strategy in various domains, e.g. for learning time series prediction with more distant time steps in the future as auxiliary targets \citep{vapnik2009new} or for learning from videos with several future frames \citep{mathieu2016deep,srivastava2016unsupervised} or representations of future frames \citep{vondrick2016anticipating} as auxiliary targets.

\section{Conclusion}
\label{sect:conclusion}

We have proposed multi-token prediction as an improvement over next-token prediction in training language models for generative or reasoning tasks. 
Our experiments (up to 7B parameters and 1T tokens) show that this is increasingly useful for larger models and in particular show strong improvements for code tasks.  
We posit that our method reduces distribution mismatch between teacher-forced training and autoregressive generation. %
When used with speculative decoding, exact inference gets 3 times faster.

In future work we would like to better understand how to automatically choose $n$ in multi-token prediction losses.
One possibility to do so is to use loss scales and loss balancing~\cite{defossez2022highfi}.
Also, optimal vocabulary sizes for multi-token prediction are likely different from those for next-token prediction, and tuning them could lead to better results, as well as improved trade-offs between compressed sequence length and compute-per-byte expenses.
Finally, we would like to develop improved auxiliary prediction losses that operate in embedding spaces~\cite{lecun2022path}.

\section*{Impact statement}
The goal of this paper is to make language models more compute and data efficient. While this may in principle reduce the ecological impact of training LLMs, we shall be careful about \emph{rebound effects}.
All societal advantages, as well as risks, of LLMs should be considered while using this work.

\section*{Environmental impact}
In aggregate, training all models reported in the paper required around 500K GPU hours of computation on hardware of type A100-80GB and H100. Estimated total emissions were around 50 tCO2eq, 100\% of which were offset by Meta’s sustainability program.

\section*{Acknowledgements}
We thank Jianyu Zhang, Léon Bottou, Emmanuel Dupoux, Pierre-Emmanuel Mazaré, Yann LeCun, Quentin Garrido, Megi Dervishi, Mathurin Videau and Timothée Darcet and other FAIR PhD students and CodeGen team members for helpful discussions. We thank Jonas Gehring for his technical expertise and the original Llama team and xFormers team for enabling this kind of research.

\clearpage

\bibliography{main}
\bibliographystyle{plainnat}

\clearpage
\newpage

\renewcommand{\thefigure}{S\arabic{figure}}
\renewcommand{\thetable}{S\arabic{table}}
\newpage
\appendix
\onecolumn

\section{Additional results on self-speculative decoding}
\label{app:decoding}
\begin{figure}[h!]
    \begin{subfigure}[b]{0.45\linewidth}
        \centering
        \includegraphics[width=0.95\linewidth]{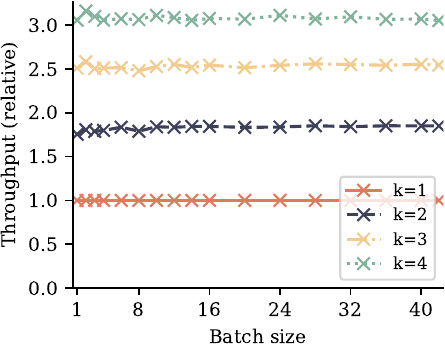}
    \end{subfigure}
    \hfill
    \begin{subfigure}[b]{0.45\linewidth}
        \centering
        \includegraphics[width=0.95\linewidth]{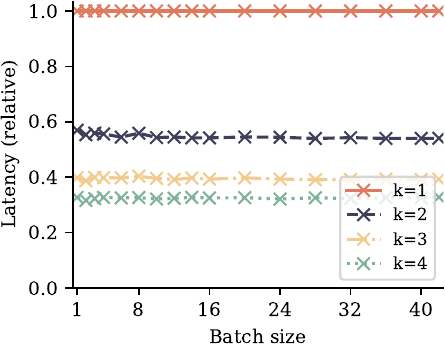}
    \end{subfigure}

    \caption{\textbf{Decoding speeds and latencies with self-speculative decoding relative to standard autoregressive decoding.} We use $k$ heads of a 4-token prediction model and evaluate decoding speeds of a code model as explained in Table~\ref{tab:decoding}. All numbers are relative to the autoregressive ($k = 1$) baseline with the same batch size.
    }
    \label{fig:decoding}
\end{figure}

\begin{table}[H]
\centering
\caption{\textbf{Relative speedups with self-speculative decoding.} 
For wikipedia and books we prompt a 7B parameter model trained on 500B tokens, and for code we prompt a 7B parameter model trained on 1T tokens of code on 4200 sequences of 512 tokens from a test dataset not seen during training, and generate completions consisting of 512 tokens using greedy self-speculative decoding  \cite{stern2018blockwise} using the indicated number of heads from a 4-token prediction model. Note that the maximal speedup that can be obtained with self-speculative decoding using $k$ heads is $k$. The last column shows the average number of tokens retrieved from a forward containing this sequence (both verification and prediction). The speedup was evaluated at the maximal batch size of 42, but is constant across batch sizes (Figure~\ref{fig:decoding}).
}
\label{tab:decoding}
\resizebox{0.9\textwidth}{!}{
\begin{tabular}{rrrrrrr}
\toprule
   & \multicolumn{2}{c}{Wikipedia} & \multicolumn{2}{c}{Books} & \multicolumn{2}{c}{Code}\\
   \cmidrule(l){2-3} \cmidrule(l){4-5} \cmidrule(l){6-7}
  \# Heads used &  Rel. speedup & Tokens / forward & Rel. speedup & Tokens / forward & Rel. speedup & Tokens / forward\\
\midrule
1 &    1.00 &          1.00 &    1.00 &          1.00 & 1.00 & 1.00\\
2 &    1.79 &          1.88 &    1.77 &          1.87 & 1.85 & 1.94\\
3 &    2.35 &          2.57 &    2.32 &          2.56 & 2.54 & 2.78 \\
4 &    2.74 &          3.12 &    2.67 &          3.09 & \textbf{3.05} & \textbf{3.50}\\
\bottomrule
\end{tabular}}
\end{table}

\begin{table}[H]
\centering
\caption{\textbf{Relative speedups with self-speculative decoding with byte-level models on code.} 
We prompt the 7B parameter models from Section~\ref{sect:byte-level} on 4096 sequences of 1024 bytes of code not seen during training, and generate completions consisting of 1024 bytes using greedy self-speculative decoding  \cite{stern2018blockwise} as in Table~\ref{tab:decoding}. The speedup was evaluated at a batch size of 16.
}
\label{tab:decoding_bytes}
\resizebox{0.9\textwidth}{!}{
\begin{tabular}{rrrrrrrrr}
\toprule
   & \multicolumn{2}{c}{$n = 8$} & \multicolumn{2}{c}{$n = 16$} & \multicolumn{2}{c}{$n = 32$} \\
   \cmidrule(l){2-3} \cmidrule(l){4-5} \cmidrule(l){6-7}
  \# Heads used &  Rel. speedup & Tokens / forward & Rel. speedup & Tokens / forward  & Rel. speedup & tokens / forward \\
\midrule
1 & 1.00 & 1.00 & 1.00 & 1.00 & 1.00 & 1.00 \\
2 &        1.94 &          1.98 &        1.94 &          1.98 &          1.93 &          1.97 \\

4 &         3.67 &          3.84 &        3.63 &          3.81 &          3.62 &          3.80 \\

8 &         6.39  &          7.04 &        6.25 &          6.92 &         6.22 &          6.89 \\

12 &  $-$    &     $-$      &      8.07 &          9.36 &       8.01 &          9.30 \\

16 &    $-$     &    $-$        &      9.24 &         11.20 &         9.15 &         11.15 \\

20 &     $-$      &   $-$         &   $-$       &     $-$       &          9.83 &         12.61 \\

24 &    $-$     &     $-$       &     $-$    &     $-$       &       10.34 &         13.67 \\

28 &   $-$       &     $-$       &   $-$     &     $-$       &        10.55 &         14.58 \\

32 &    $-$      &    $-$        &   $-$     &     $-$       &       10.84 &         15.35 \\
\bottomrule
\end{tabular}}
\end{table}

\newpage

\section{Alternative architectures}
\label{app:architecture}
\begin{table*}[ht!]
    \centering
\caption{\textbf{Alternative architectures improve on baseline but not as consistently.} Alternative architectures for multi-token prediction are worth exploring to improve efficiency. Here we tried Anticausal, causal and linear and showed no significant improvement with respect to Parallel architecture.}
\label{tab:extra_archs}
\resizebox{.95\textwidth}{!}{
\begin{tabular}{lllrrrrrrrrrr}
\toprule
 &  &  &  & \multicolumn{3}{c}{MBPP} & \multicolumn{3}{c}{HumanEval} & \multicolumn{3}{c}{APPS/Intro} \\
$n$ & Head type & Architecture & +Layers &  @1 & @10 & @100 & @1 & @10 & @100 & @1 & @10 & @100 \\
\midrule
1 & transformer & parallel & 0 & 30.0 & 53.8 & 73.7 & 22.8 & 36.4 & 62.0 & 2.8 & 7.8 & 17.4 \\
\midrule
\multirow[c]{5}{*}{4} & linear & parallel & 0 & 33.6 & 55.0 & 76.2 & 21.9 & 38.5 & 63.7 & 3.1 & 10.1 & 23.0 \\
\cmidrule{2-13}
 & \multirow[c]{4}{*}{transformer} & anticausal & 0 & 30.8 & 54.8 & 75.3 & 20.9 & 38.4 & 64.5 & 2.0 & 8.7 & 21.6 \\
\cmidrule{3-13}
 &  & causal & 0 & 31.9 & 54.9 & 74.9 & 20.9 & 38.1 & 67.3 & 4.0 & 11.6 & 22.8 \\
\cmidrule{3-13}
 &  & \multirow[c]{2}{*}{parallel} & 0 & 33.8 & 55.9 & 76.9 & 24.0 & 40.1 & 66.1 & 1.6 & 7.1 & 19.9 \\
 &  &  & 3 & 33.3 & 55.7 & 77.3 & 22.4 & 39.4 & 66.7 & 2.6 & 9.5 & 22.1 \\
\bottomrule
\end{tabular}
}
\end{table*}

The architecture described in Section~\ref{sect:method} is not the only sensible option, but proved technically viable and well-performing in our experiments. We describe and compare alternative architectures in this section.
\paragraph{Replicated unembeddings}
Replicating the unembedding matrix $n$ times is a simple method for implementing multi-token prediction architectures. However, it requires matrices with shapes $(d, nV)$ in the notation of Section~\ref{sect:method}, which is prohibitive for large-scale trainings. 

\paragraph{Linear heads}
Apart from using a single transformer layer for the heads $H_i$, other architectures are conceivable. We experimented with a single linear layer without any nonlinearity as heads, amounting to linear probing of the model's residual representation $z$. Architectures with more than one layer per head are also possible, but we did not pursue this direction further.

\paragraph{Causal and anticausal variant}
Instead of making the prediction heads $P_i(x_{t+i}\,|\,z_{t:1})$ architecturally independent of each other, we can also allow them to rely on other heads' (pre-unembedding) outputs. In a \emph{causal} variant, later prediction heads are applied on top of the previous ones, i.e. the $i$-th prediction head $P_i$ is given by 
\[ P_\theta(x_{t+i}|\cdot) = \mathrm{softmax} \circ f_u \circ f_{h_i} \circ f_{h_{i-1}} \cdots \circ f_{h_1} \circ f_s. \]
In another \emph{anticausal} variant, the network starts by predicting the most distant tokens before gradually refining up to the following token:
\[ P_\theta(x_{t+i}|\cdot) = \mathrm{softmax} \circ f_u \circ f_{h_i} \circ f_{h_{i+1}} \cdots \circ f_{h_{n}} \circ f_s. \]
These architectures likewise allow a sequential forward/backward order as the parallel architecture from Section~\ref{sect:method}. This is described in Figure~\ref{fig:backward_causal}.

\begin{figure*}[ht!]
    \centering
    \includegraphics{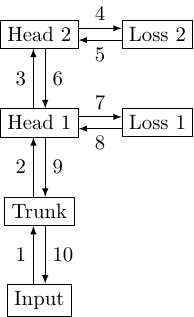}

    \caption{\textbf{Order of the forward/backward in a causal $n$-token prediction model with $n = 2$ heads.} Like in the forward/backward depicted for parallel prediction heads in Figure~\ref{fig:backward}, we avoid materializing all unembedding layer gradients in memory simultaneously and reduce peak GPU memory usage significantly. The iteration over the heads starts with the one furthest to the trunk. At each head, a gradient from the succeeding prediction heads and from the head's own loss are accumulated for both the head's output and its weights.
    }
    \label{fig:backward_causal}
\end{figure*}

\newpage
\section{Training speeds}
\begin{table}[!h]
    \centering
        \caption{\textbf{Training time relative to next-token prediction training.} The slight overhead when using multi-token prediction here is explained by a suboptimal use of Fully Sharded Data Parallel. In our implementation, when doing separate backward passes for each head, we lose the overlap of layer weight communication and computation, therefore it incurs a very slight overhead that can be removed if reimplemented correctly.}
    \label{tab:wps}
    \begin{tabular}{lrrr}
        \toprule
        Model & n=1 & n=2 & n=4 \\
        \midrule
        0.3B & 1.00 & 1.07 & 1.22 \\
        0.6B & 1.00 & 1.05 & 1.13 \\
        1.3B & 1.00 & 1.04 & 1.12 \\
        3B & 1.00 & 1.02 & 1.07 \\
        6.7B & 1.00 & 1.02 & 1.07 \\
        13B & 1.00 & 1.04 & 1.09 \\
        \bottomrule
    \end{tabular}
\end{table}

\section{Finetuning}
\begin{table*}[ht!]
    \centering
    \caption{\textbf{Finetuning LLama 2 with multi-token prediction does not significantly improve performance.} We tried to finetune LLama 2 with 4-token prediction but this did not yield significant improvements compared to the baseline. We suppose that this new loss changes the initialization too brutally and never really recovers. We still some improvements for example on MBPP Pass@1. All runs use 200B tokens of code.}
\label{tab:finetune_llama}
\resizebox{.95\textwidth}{!}{
\begin{tabular}{lllrrrrrrrrr}
\toprule
   &  &   & \multicolumn{3}{c}{MBPP} & \multicolumn{3}{c}{HumanEval} & \multicolumn{3}{c}{APPS/Intro} \\
$n$ & Head type  & +Layers & @1 & @10 & @100 & @1 & @10 & @100 & @1 & @10 & @100 \\
\midrule
 1 & transformer  & 0 & 39.6 & 65.1 & 82.4 & 31.4 & 57.7 & 84.7 & 10.0 & 21.6 & 36.7 \\
\cmidrule{1-12}
   \multirow[c]{3}{*}{4} & linear  & 0 & 39.3 & 63.7 & 81.3 & 29.0 & 53.4 & 82.2 & 6.9 & 20.0 & 34.0 \\
 \cmidrule{2-12}
   & \multirow[c]{2}{*}{transformer}  & 0 & 38.3 & 62.2 & 80.1 & 27.9 & 53.6 & 82.4 & 5.8 & 18.2 & 34.3 \\
    &  & 3 & 42.5 & 64.4 & 81.3 & 28.7 & 56.9 & 82.4 & 7.8 & 21.2 & 37.3 \\
\bottomrule
\end{tabular}
}
\end{table*}

\newpage
\section{Additional results on model scaling behavior}
\label{app:model-scaling}
\begin{table*}[ht!]
    \centering
\caption{\textbf{Scaling model size} Full results of scaling model size with n=1,2 and 4.}
\label{tab:model_scaling}
\begin{tabular}{llrrrrrr}
\toprule
 &  & \multicolumn{3}{c}{MBPP} & \multicolumn{3}{c}{HumanEval} \\
Model Size & Fut & @1 & @10 & @100 & @1 & @10 & @100 \\
\midrule
\multirow[c]{3}{*}{0.3B} & 1 & 1.8 & 10.4 & 29.9 & 1.9 & 5.0 & 10.9 \\
 & 2 & 1.7 & 10.1 & 27.2 & 1.5 & 4.4 & 10.3 \\
 & 4 & 1.0 & 6.3 & 20.1 & 1.2 & 4.0 & 8.6 \\
\cmidrule{1-8}
\multirow[c]{3}{*}{0.6B} & 1 & 4.7 & 21.0 & 45.2 & 2.9 & 8.5 & 16.7 \\
 & 2 & 4.6 & 21.0 & 44.7 & 3.2 & 8.9 & 16.2 \\
 & 4 & 3.0 & 15.6 & 38.0 & 2.7 & 7.7 & 15.5 \\
\cmidrule{1-8}
\multirow[c]{3}{*}{1.3B} & 1 & 6.8 & 27.0 & 51.0 & 4.6 & 13.1 & 24.3 \\
 & 2 & 7.3 & 27.5 & 51.7 & 5.4 & 13.6 & 23.3 \\
 & 4 & 7.4 & 27.6 & 50.1 & 4.8 & 12.3 & 22.5 \\
\cmidrule{1-8}
\multirow[c]{3}{*}{3B} & 1 & 11.1 & 36.4 & 60.4 & 7.2 & 17.2 & 29.8 \\
 & 2 & 11.8 & 37.2 & 60.5 & 8.0 & 18.2 & 31.2 \\
 & 4 & 12.7 & 37.6 & 61.1 & 7.2 & 18.5 & 33.3 \\
\cmidrule{1-8}
\multirow[c]{3}{*}{6.7B} & 1 & 23.9 & 54.2 & 74.7 & 12.8 & 29.3 & 51.7 \\
 & 2 & 24.7 & 54.8 & 76.4 & 13.2 & 32.2 & 53.9 \\
 & 4 & 26.0 & 55.8 & 76.0 & 13.8 & 33.2 & 58.5 \\
\cmidrule{1-8}
\multirow[c]{3}{*}{13B} & 1 & 26.0 & 57.1 & 77.0 & 14.1 & 33.6 & 56.0 \\
 & 2 & 30.5 & 60.5 & 79.4 & 15.2 & 36.9 & 60.0 \\
 & 4 & 30.5 & 61.0 & 79.2 & 15.8 & 38.6 & 63.5 \\
\bottomrule
\end{tabular}
\end{table*}

\section{Details on CodeContests finetuning}
\label{app:finetuning}
We use the Python subset of the CodeContests \citep{li2022competition} train split with reward annotations (``correct'' / ``incorrect'') and condition on correct solutions at evaluation time. For evaluation, we generate 1000 samples per problem from the test split for each temperature $T \in \{0.5, 0.6, 0.7, 0.8, 0.9\}$, and compute the unbiased estimator for pass@k  from \citet{chen2021evaluating} for each value of $k$ and $T$. It is possible that models that were pretrained with different losses have different respective optimal temperatures for pass@k, so we compute and show $k \mapsto \max_T \mathrm{pass\_at}(k, T)$ in Figure~\ref{fig:dm_contests}. In other words, we grant pass@k access to a temperature oracle. For small values of $k$, pass@k measures the capability of understanding and solving tasks while for large $k$, it additionally favors diversity in outputs. According to the results in Figure~\ref{fig:dm_contests}, multi-token prediction pretraining leads to finetuned models that are better on both axes.

\newpage
\section{Additional results on natural language benchmarks}
\label{app:choice_nlp}
We evaluate the models from Section~\ref{sect:nlp} on standard natural language processing benchmarks: ARC Challenge \citep{yadav2019quick}, COPA \citep{roemmele2011choice}, Hellaswag \citep{zellers2019hellaswag}, Natural Questions \citep{kwiatkowski2019natural}, PIQA \citep{bisk2019piqa}, SIQA \citep{sap2019socialiqa} and TriviaQA \citep{joshi2017triviaqa}.

\begin{figure}[h!]

    \centering
    \includegraphics[width=0.8\columnwidth]{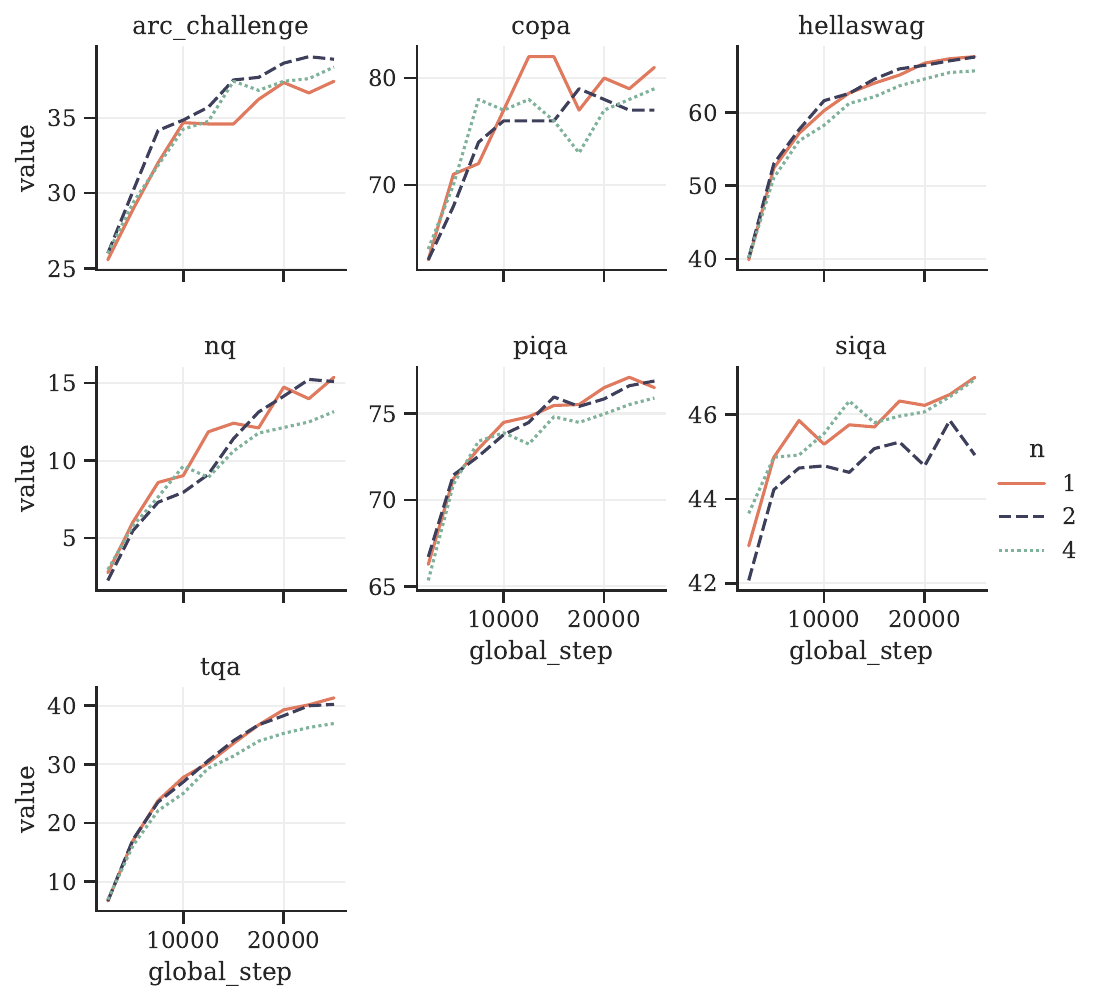}
    \caption{\textbf{Multiple token training with 7B models doesn't improve performance on choice tasks.} This figure shows the evolution of average accuracy of some standard NLP benchmarks (ARC Challenge COPA Hellaswag MMLU Natural Questions PIQA SIQA and TriviaQA. For the 7B models trained on 200B tokens of language data, the 2 future token model has the same performance as the baseline and the 4 future token model regresses a bit. Larger model sizes might be necessary to see improvements on these tasks.
    }
    \label{fig:nlp_evol_all}
\end{figure}

\newpage
\section{Additional results on abstractive text summarization}
\label{app:summarzation_all}
\begin{table*}[b!]
\centering
\caption{\textbf{Comprehensive evaluation on abstractive text summarization.} 
ROUGE-n (n-gram overlap) and ROUGE-L (longest common subsequence overlap) 
$F_1$ scores for 7B models trained on 200B and 500B tokens of natural language, respectively. The last three columns correspond to models trained on 500B tokens, the previous three to models trained on 200B tokens. Shown are numbers of the $n=1$ baseline and the absolute difference of $n=2$ and $n=4$ models trained on the same number of tokens. Summary-level ROUGE-L (``ROUGE-L\textsubscript{sum}'') is reported where it differs from ROUGE-L. Model checkpoints with maximal validation ROUGE-L $F_1$ are selected separately for each model dataset and model type and reported in the first row corresponding to each dataset. Boldface for numbers within 0.05 difference to the best one for each dataset size separately.}

\label{tab:summarization_all}
\resizebox{0.9\linewidth}{!}{
\begin{tabular}{llrrrrrr}
\toprule
      Task &        Metric &  Baseline 200B &  $\Delta_{n=2}$ &  $\Delta_{n=4}$ &  Baseline 500B &  $\Delta_{n=2}$ &  $\Delta_{n=4}$  \\
\midrule
\multirow{6}{*}{CNN/Dailymail \citep{nallapati2016abstractive}} 
 & \textit{evaluation epoch} &  2 &  2 &  2 & 2 & 2 & 2 \\
 &    ROUGE-1 & 42.88 &     \textbf{+0.74} &    \textbf{+0.74} & 43.77 &         \textbf{+0.55} &         \textbf{+0.50} \\
 &    ROUGE-2 & 19.56 &     \textbf{+0.52} &    \textbf{+0.53} &  20.34 &         \textbf{+0.52} &         +0.34 \\
 &    ROUGE-3 & 11.11 &     \textbf{+0.39} &    \textbf{+0.35} &  11.69 &         \textbf{+0.36} &         +0.19 \\
 &    ROUGE-L & 29.72 &     \textbf{+0.66} &     +0.49 &  30.51 &         \textbf{+0.48} &         +0.37 \\
 & ROUGE-Lsum & 40.18 &     \textbf{+0.72} &    \textbf{+0.68} &  41.02 &         \textbf{+0.56} &         \textbf{+0.52} \\
 \midrule
\multirow{5}{*}{Multi-News \citep{fabbri2019multinews}} 
 & \textit{evaluation epoch} & 1 &  3 &  3 &  2 &  3 &  2 \\
 &    ROUGE-1 & 44.48 &     \textbf{+1.70} &     \textbf{+1.72} &  45.87 &         \textbf{+1.05} &         +0.69 \\
 &    ROUGE-2 & 16.88 &     +0.44 &     \textbf{+0.70} &  17.56 &         \textbf{+0.42} &         \textbf{+0.40} \\
 &    ROUGE-3 &  9.63 &     -0.06 &     \textbf{+0.17} &   9.91 &         \textbf{+0.22} &         \textbf{+0.18} \\
 &    ROUGE-L & 23.82 &     +0.17 &     \textbf{+0.40} &  24.22 &         +0.20 &         \textbf{+0.26} \\
 \midrule
\multirow{5}{*}{OrangeSum \citep{eddine2021barthez}} 
 & \textit{evaluation epoch} &  2 &  2 &  3 &  2 &  1 &  3 \\
 &    ROUGE-1 & 32.95 &     \textbf{+0.41} &     +0.35 &  33.37 &         +0.32 &         \textbf{+0.78} \\
 &    ROUGE-2 & 13.90 &     \textbf{+0.31} &     \textbf{+0.36} &  14.22 &         +0.25 &         \textbf{+0.53} \\
 &    ROUGE-3 &  8.01 &     \textbf{+0.19} &     \textbf{+0.21} &   8.12 &         +0.22 &         \textbf{+0.48} \\
 &    ROUGE-L & 23.62 &     +0.36 &     \textbf{+0.51} &  23.91 &         +0.23 &         \textbf{+0.66} \\
 \midrule
\multirow{5}{*}{pn-summary \citep{Farahani_2021} } 
 & \textit{evaluation epoch} &  1 &  1 &  1 &  1 &  2 &  3 \\
 &    ROUGE-1 &  \textbf{1.03} &     \textbf{+0.02} &     \textbf{0.00} &   0.92 &         \textbf{+0.09} &         \textbf{+0.05} \\
 &    ROUGE-2 &  \textbf{0.13} &     \textbf{+0.02} &     \textbf{+0.03} &   \textbf{0.15} &         \textbf{0.00} &         \textbf{0.00} \\
 &    ROUGE-3 &  \textbf{0.02} &      \textbf{0.00} &     \textbf{+0.02} &   \textbf{0.02} &         \textbf{0.00} &         \textbf{+0.02} \\
 &    ROUGE-L &  \textbf{1.02} &     \textbf{+0.03} &     \textbf{+0.01} &   0.91 &         \textbf{+0.09} &         \textbf{+0.05} \\
 \midrule
\multirow{5}{*}{SAMSum \citep{Gliwa_2019}} 
 & \textit{evaluation epoch} & 3 & 3 & 3 &  3 &  3 &  3 \\
 &    ROUGE-1 & 51.39 &     \textbf{+0.70} &     +0.63 &  52.54 &        -0.24 &         \textbf{+0.69}  \\
 &    ROUGE-2 & 26.46 &     \textbf{+0.76} &     +0.30 &  27.74 &        -0.20 &         \textbf{+0.82} \\
 &    ROUGE-3 & 16.40 &     \textbf{+0.91} &     +0.28 &  17.56 &        -0.30 &         \textbf{+0.71} \\
 &    ROUGE-L & 42.59 &     \textbf{+0.90} &     +0.51 &  43.92 &        -0.10 &         \textbf{+0.63} \\
 \midrule
\multirow{5}{*}{ThaiSum \citep{chumpolsathien_2020}} 
 & \textit{evaluation epoch} & 2 & 3 & 3 &  3 &  3 &  3 \\
 &    ROUGE-1 & 45.08 &     +0.63 &     \textbf{+1.12} &  45.48 &         +0.77 &         \textbf{+0.91} \\
 &    ROUGE-2 & 27.85 &     +0.30 &     \textbf{+0.73} &  28.07 &         \textbf{+0.74} &         +0.64 \\
 &    ROUGE-3 & 15.73 &     +0.04 &     \textbf{+0.43} &  15.82 &         \textbf{+0.50} &         +0.30 \\
 &    ROUGE-L & 44.92 &     +0.64 &     \textbf{+1.12} &  45.31 &         +0.76 &         \textbf{+0.89} \\
 \midrule
\multirow{5}{*}{WikiSummary \citep{Bert2BertWikiSummaryPersian}} 
 & \textit{evaluation epoch} & 3 & 3 & 3 &  3 &  3 &  3 \\
 &    ROUGE-1 & 10.16 &    \textbf{+0.67} &    -0.23 &  \textbf{12.80} &        -0.17 &        -0.99 \\
 &    ROUGE-2 &  \textbf{4.46} &    -0.03 &    -0.09 &   \textbf{6.17} &        -0.11 &        -0.69 \\
 &    ROUGE-3 &  1.31 &    \textbf{+0.21} &    +0.13 &   \textbf{1.98} &        -0.08 &        -0.33 \\
 &    ROUGE-L & 10.11 &    \textbf{+0.65} &    -0.28 &  \textbf{12.69} &        -0.17 &        -0.99 \\
 \midrule
\multirow{5}{*}{XSum \citep{narayan2018dont}} 
 & \textit{evaluation epoch} & 2 & 2 & 3 &  2 &  2 &  3 \\
 &    ROUGE-1 & 42.16 &     +0.71 &     \textbf{+1.07} &  43.42 &         \textbf{+0.78} &         +0.67 \\
 &    ROUGE-2 & 19.19 &     \textbf{+0.54} &     \textbf{+0.55} &  20.32 &         \textbf{+0.68} &         +0.34 \\
 &    ROUGE-3 & 10.43 &     \textbf{+0.38} &     +0.28 &  11.23 &         \textbf{+0.48} &         +0.20 \\
 &    ROUGE-L & 34.03 &     +0.67 &     \textbf{+0.92} &  35.18 &         \textbf{+0.79} &         +0.63 \\
\bottomrule
\end{tabular}
}

\end{table*}

In this section, we report comprehensive evaluation results on summarization tasks for the 7B parameter models trained on 200B and 500B tokens of natural language from Section~\ref{sect:nlp}.

\begin{table}[ht!] 
\centering
\caption{\textbf{Performance on abstractive text summarization.} ROUGE-L (longest common subsequence overlap) $F_1$ score for 7B models trained on 200B and 500B tokens of natural language. We finetune the respective models on each task's training data separately for a given number of epochs and select the checkpoints with maximal ROUGE-L $F_1$ on the validation dataset. The second and fifth column report the numbers for a next-token prediction model, while the third, fourth, sixth and seventh one report the absolute improvements for 2-token and 4-token prediction models trained on the same amount of data, respectively.
Boldface for numbers within 0.05 difference to the best one for each dataset size separately.}
\label{tab:summarization}
\begin{tabular}{lrrrrrr}
\toprule
      Dataset &     Baseline 200B &  $\Delta_{n=2}$ &  $\Delta_{n=4}$ & Baseline 500B & $\Delta_{n=2}$ & $\Delta_{n=4}$ \\
\midrule
CNN/Dailymail &  29.72 &         \textbf{+0.66} &         +0.49 &  30.51 &         \textbf{+0.48} &         +0.37 \\
Multi-News &  23.82 &         +0.17 &         \textbf{+0.40} &  24.22 &         +0.20 &         \textbf{+0.26} \\
OrangeSum &  23.62 &         +0.36 &         \textbf{+0.51} &  23.91 &         +0.23 &         \textbf{+0.66} \\
pn-summary &   \textbf{1.02} &         \textbf{+0.03} &         \textbf{+0.01} &   0.91 &         \textbf{+0.09} &         \textbf{+0.05} \\
SAMSum &  42.59 &         \textbf{+0.90} &         +0.51 &  43.92 &        -0.10 &         \textbf{+0.63} \\
ThaiSum &  44.92 &         +0.64 &         \textbf{+1.12} &  45.31 &         +0.76 &         \textbf{+0.89} \\
WikiSummary &  10.11 &         \textbf{+0.65} &        -0.28 &  \textbf{12.69} &        -0.17 &        -0.99 \\
XSum &  34.03 &         +0.67 &         \textbf{+0.92} &  35.18 &         \textbf{+0.79} &         +0.63 \\
         \midrule
\emph{Average} &  26.23 & \textbf{+0.51} & \textbf{+0.46} & 27.08 & \textbf{+0.28} & \textbf{+0.31} \\
\bottomrule
\end{tabular}
\end{table}

\begin{table*}[ht!]
\centering
\caption{\textbf{Summary statistics for abstractive text summarization evaluations.} 
Reported are averages for ROUGE-n and ROUGE-L metrics across all datasets from Table~\ref{tab:summarization_all}, separately for precision, recall and $F_1$ score. Both 2-token and 4-token prediction models outperform the next-token prediction baseline. Trained on 500B tokens, 4-token prediction models appear better at recall metrics while 2-token prediction models appear better at precision metrics. Model checkpoints are selected as described in Table~\ref{tab:summarization_all}. Boldface for numbers within 0.05 difference to the best one for each dataset size separately.}
\label{tab:summarization_avgs}
\begin{tabular}{llrrrrrr}
\toprule
Metric & Aspect & Baseline 200B & $\Delta_{n=2}$ & $\Delta_{n=4}$ & Baseline 500B & $\Delta_{n=2}$ & $\Delta_{n=4}$ \\
\midrule
\multirow{3}{*}{ROUGE-1} & $F_1$ &  33.77 &         \textbf{+0.70} &         \textbf{+0.68} &  34.77 &         \textbf{+0.39} &         \textbf{+0.41} \\
           & precision &  35.76 &         \textbf{+0.88} &         \textbf{+0.83} &  37.03 &         \textbf{+0.42} &        -0.04 \\
           & recall &  34.37 &         \textbf{+0.45} &         \textbf{+0.45} &  35.14 &         +0.35 &         \textbf{+0.68} \\
\midrule
\multirow{3}{*}{ROUGE-2} & $F_1$ &  16.06 &         \textbf{+0.36} &         \textbf{+0.39} &  16.82 &         \textbf{+0.29} &         \textbf{+0.30} \\
           & precision &  16.97 &         \textbf{+0.40} &         \textbf{+0.43} &  17.91 &         \textbf{+0.29} &         +0.03 \\
           & recall &  16.34 &         +0.28 &         \textbf{+0.35} &  16.99 &         +0.32 &         \textbf{+0.48} \\
\midrule
\multirow{3}{*}{ROUGE-3} & $F_1$ &   9.08 &         \textbf{+0.26} &         \textbf{+0.23} &   9.54 &         \textbf{+0.18} &         \textbf{+0.22} \\
           & precision &   9.59 &         \textbf{+0.29} &         \textbf{+0.28} &  10.17 &         \textbf{+0.18} &         +0.05 \\
           & recall &   9.26 &         \textbf{+0.21} &         \textbf{+0.20} &   9.65 &         +0.21 &         \textbf{+0.35} \\
\midrule
\multirow{3}{*}{ROUGE-L} & $F_1$ &  26.23 &         \textbf{+0.51} &         \textbf{+0.46} &  27.08 &         \textbf{+0.28} &         \textbf{+0.31} \\
           & precision &  27.79 &         \textbf{+0.62} &         +0.55 &  28.85 &         \textbf{+0.28} &        -0.09 \\
           & recall &  26.71 &         \textbf{+0.37} &         \textbf{+0.32} &  27.40 &         +0.28 &         \textbf{+0.57} \\
\midrule
\multirow{3}{*}{ROUGE-L\textsubscript{sum}} & $F_1$ &  27.53 &         \textbf{+0.52} &         \textbf{+0.48} &  28.40 &         \textbf{+0.29} &         \textbf{+0.33} \\
           & precision &  29.07 &         \textbf{+0.64} &         +0.58 &  30.15 &         \textbf{+0.29} &        -0.08 \\
           & recall &  28.13 &         \textbf{+0.35} &         \textbf{+0.33} &  28.81 &         +0.29 &         \textbf{+0.60} \\
\bottomrule
\end{tabular}
\end{table*}

\clearpage
\section{Additional results on mathematical reasoning in natural language}
\label{app:gsm8k}
\begin{figure}[h!]
    \centering
    \includegraphics[width=0.8\columnwidth]{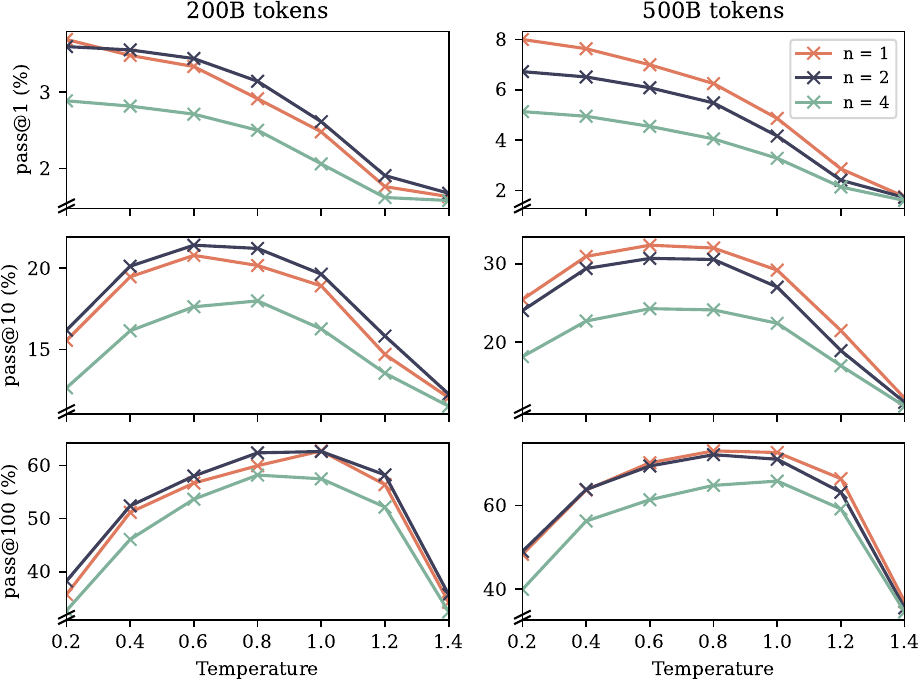}
    \caption{\textbf{Performance on the mathematical reasoning benchmark GSM8K \citep{cobbe2021training}.} We evaluate pretrained next-token and multi-token prediction models trained on 200B and 500B tokens of natural language in 8-shot mode using nucleus sampling \citep{holtzman2020curious} with probability mass 0.95 and various sampling temperatures. Reported are the frequencies of the correct final answer to appear among $k$ samples, for $k=1, 10, 100$, estimated from 200 samples like in code generation benchmarks \citep{chen2021evaluating}. After 200B tokens, the 2-token prediction model has a clear advantage over the next-token baseline but the order reverses after 500B tokens. The 4-token prediction model is worse throughout. We interpret this similarly to the findings in Section~\ref{sect:induction}: the follow-your-nose chains-of-thought required for GSM8K may be difficult to learn from a limited amount of data, attesting to the data efficiency of multi-token prediction training. Once the correct circuits for correct autoregressive chains-of-thought in this domain have formed, however, multi-token prediction comes at a cost.
    }
    \label{fig:gsm8k}
\end{figure}

\newpage
\section{Additional results on induction learning}
\label{app:induction}
\begin{figure}[h!]
    \centering
    \includegraphics{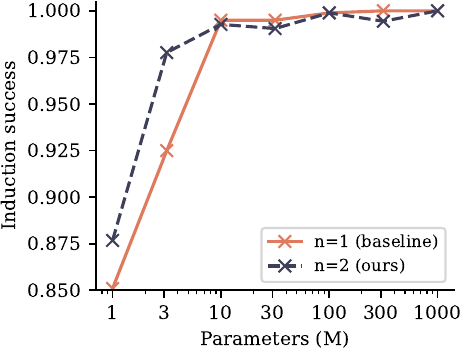}
    \caption{\textbf{Induction capability of $n$-token prediction models trained on higher-quality data.}
    Shown is accuracy on the second token of two token names that have already been mentioned previously. Training on a 9:1 mix of a books dataset and the children storiy dataset, we observe that induction capability forms significantly earlier in training (not shown here) and to a higher degree. We believe that this is explained both because our evaluation dataset no longer contains out-of-distribution tokens (Section~\ref{sect:induction}) and because the higher-quality data contained in the books dataset makes induction necessary earlier on (especially for small models, cf. \citet{singh2023transient}). In particular, by enforcing the formation of induction capability in the model by means of the dataset -- instead of the loss -- the advantage of 2-token prediction models on this task disappears except for the smallest models: feature learning converts the task into a pure next-token prediction task.
    }
    \label{fig:induction_b3g}
\end{figure}

\newpage
\section{Additional results on algorithmic reasoning}
\label{app:poly}
We investigate the following \emph{computation-sharing hypothesis} for explaining the efficacy of multi-token prediction as training loss.
\begin{quote}
    The prediction difficulty of different tokens in natural text varies greatly. Some tokens may be the continuations of partial words that are uniquely determined from their preceding context without any effort, while others may require to predict theorem names in difficult mathematical proofs or the correct answer to an exam question. Language models with residual connections have been shown to refine their output token distribution with each successive layer, and can be trained with early exit strategies that spend variable amounts of computational resources per token position. Multi-token prediction losses explicitly encourage information-sharing between adjacent token positions and can thus be viewed as a method to learn allocating computational resources in language models more efficiently to the tokens that benefit most of it.
\end{quote}
To check the truth of this hypothesis, we augment the polynomial arithmetic task from Section~\ref{sect:algorithmic} with a varying number of \emph{pause tokens} \citep{goyal2023think} inserted between the question and a token that denotes the beginning of the answer. Pause tokens introduce additional computational resources that can be expended for computations that are expected to be useful later on in the sequence, in other words: to start thinking about the answer. According to the \emph{computation-sharing hypothesis}, multi-token prediction models learn information-sharing and thus computation-sharing between token positions more easily, and may be better at making use of these additional computational resources than next-token prediction models are. In Figure~\ref{fig:poly_pause}, we show the evaluation results on the polynomial arithmetic task with a fixed number of pause tokens inserted both at training and evaluation time. Multi-token prediction models likewise outperform next-token prediction models on these task variants across task difficulties and model sizes. However, we do not see strong evidence of a widening or shrinking of this gap i.e. we cannot conclude from these experiments on the veracity of the computation-sharing hypothesis.

In Table~\ref{tab:pause}, we report results from another experiment in the same spirit: by adding spaces and newlines to HumanEval and MBPP prompts, we add ``pause tokens'' in a somewhat natural way. According to these results, multi-token prediction models have a slight advantage at using this additionally provided compute, but the effect is marginal. 

\begin{figure*}[b!]
    \begin{subfigure}{0.5\textwidth}
        \centering
        \includegraphics{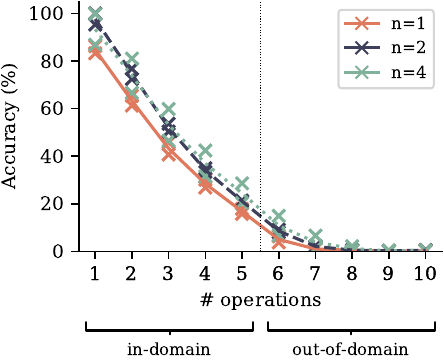}
        \label{fig:poly_5}
        \caption{5 pause tokens}
    \end{subfigure}
    \hfill
    \begin{subfigure}{0.5\textwidth}
        \centering
        \includegraphics{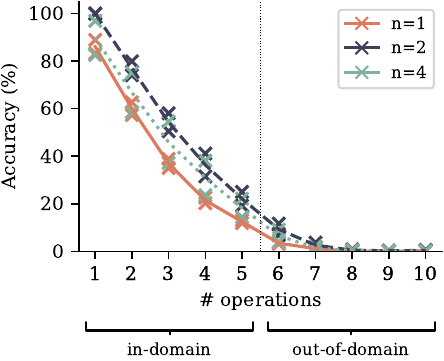}
        \label{fig:poly_10}
        \caption{10 pause tokens}
    \end{subfigure}

    \caption{\textbf{Accuracy on a polynomial arithmetic task with varying number of operations per expression and pause tokens.} We train and evaluate models on the polynomial arithmetic task described in Section~\ref{sect:algorithmic}, modified by the addition of \emph{pause tokens} \citep{goyal2023think}: between the question and the equality sign that indicates the beginning of the answer, we add a constant number of pause tokens both in training and evaluation. For both a variant with five and with ten pause tokens, respectively, we observe comparable improvements from using multi-token prediction to the ones obtained in the case without pause tokens (Figure~\ref{fig:poly_0}).
    }
    \label{fig:poly_pause}
\end{figure*}

\begin{figure*}[b!]
    \includegraphics{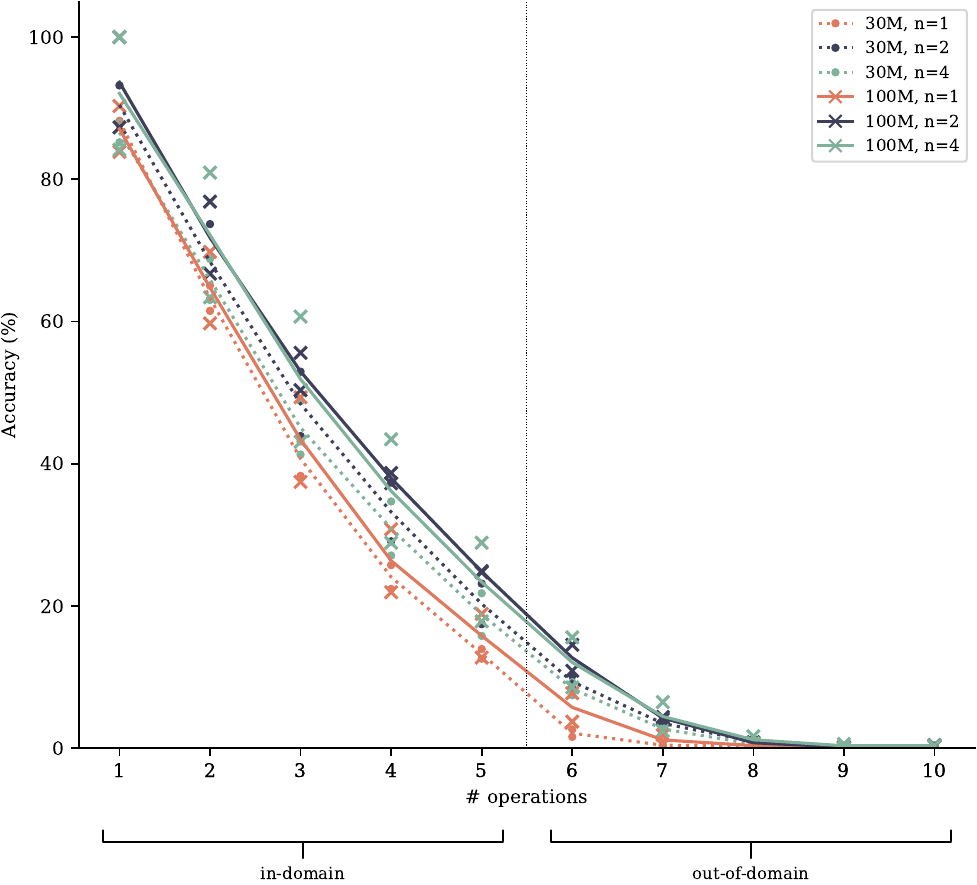}

    \caption{\textbf{Accuracy on a polynomial arithmetic task for two model sizes.} We train and evaluate models with 30M and 100M parameters on the polynomial arithmetic task described in Section~\ref{sect:algorithmic}. Tripling the model size has a smaller effect on performance than replacing next-token prediction loss by multi-token prediction. Shown are two independent runs per configuration and their means, the 100M parameter models being identical to the ones in Figure~\ref{fig:poly_0}.
    }
    \label{fig:poly_scale}
\end{figure*}

\begin{table}[H]
\centering
\caption{\textbf{Utilization of additional whitespace tokens in code benchmarks.}}
\label{tab:pause}
\begin{tabular}{llrr}
\toprule
Task  & Whitespace & $n = 1$      &  $n = 4$     \\
\midrule
APPS/Intro & spaces + newline &  +0.21 &  \textbf{+0.34} \\
APPS/Intro & newline  &  \textbf{+0.79} &  +0.69 \\
HumanEval & spaces + newline & -0.72 & \textbf{-0.16} \\
HumanEval & newline  & -0.26 &  \textbf{+0.10} \\
MBPP & spaces + newline & -0.10 & \textbf{-0.06} \\
MBPP & newline  &  \textbf{+0.03} & -0.08 \\
\midrule
\emph{Average} & & -0.01 & \textbf{+0.14} \\
\bottomrule
\end{tabular}
\end{table}

\newpage
\begin{table*}[ht!]
    \centering
    \caption{\textbf{Optimal temperatures for all numbers in table ~\ref{tab:future_tokens_results}} }
    \label{tab:optimal_temps}
\begin{tabular}{cccccccccccc}
\toprule
  \multirow[c]{2}{*}{Training data} & \multirow[c]{2}{*}{Vocabulary} & \multirow[c]{2}{*}{n} & \multicolumn{3}{c}{MBPP} & \multicolumn{3}{c}{HumanEval} & \multicolumn{3}{c}{APPS/Intro} \\
 \cmidrule(l){4-6}\cmidrule(l){7-9} \cmidrule(l){10-12}
 &  &  & @1 & @10 & @100 & @1 & @10 & @100 & @1 & @10 & @100 \\
\midrule
\multirow[c]{4}{*}{\shortstack{313B bytes\\(0.5 epochs)}} & \multirow[c]{4}{*}{bytes} & 1 & 0.2 & 0.8 & 0.8 & 0.1 & 0.8 & 0.8 & 0.8 & 0.8 & 0.8 \\
 &  & 8 & 0.1 & 0.8 & 0.8 & 0.1 & 0.8 & 0.8 & 0.4 & 0.4 & 0.4 \\
 &  & 16 & 0.1 & 0.8 & 0.8 & 0.1 & 0.8 & 0.8 & 0.4 & 0.4 & 0.4 \\
 &  & 32 & 0.1 & 0.4 & 0.8 & 0.1 & 0.4 & 0.8 & 0.1 & 0.4 & 0.4 \\
 \multirow[c]{5}{*}{\shortstack{200B tokens\\(0.8 epochs)}} & \multirow[c]{5}{*}{32k tokens} & 1 & 0.1 & 0.8 & 0.8 & 0.1 & 0.8 & 0.8 & 0.1 & 0.4 & 0.8 \\
 &  & 2 & 0.1 & 0.8 & 0.8 & 0.2 & 0.8 & 0.8 & 0.4 & 0.4 & 0.8 \\
 &  & 4 & 0.1 & 0.8 & 0.8 & 0.1 & 0.8 & 0.8 & 0.2 & 0.8 & 0.8 \\
 &  & 6 & 0.1 & 0.8 & 0.8 & 0.2 & 0.8 & 0.8 & 0.4 & 0.4 & 0.8 \\
 &  & 8 & 0.1 & 0.8 & 0.8 & 0.1 & 0.8 & 0.8 & 0.2 & 0.4 & 0.8 \\
\multirow[c]{2}{*}{\shortstack{1T tokens\\(4 epochs)}} & \multirow[c]{2}{*}{32k tokens} & 1 & 0.1 & 0.8 & 0.8 & 0.1 & 0.8 & 0.8 & 0.1 & 0.4 & 0.8 \\
 &  & 4 & 0.1 & 0.8 & 0.8 & 0.2 & 0.8 & 0.8 & 0.4 & 0.8 & 0.8 \\
\bottomrule
\end{tabular}
\end{table*}

\section{Additional intuitions on multi-token prediction}
\label{app:intuitions}

\subsection{Comparison to scheduled sampling}
In Section~\ref{sect:mismatch}, we argued that multi-token prediction reduces the distribution mismatch between teacher-forced training and autoregressive evaluation of language models. Scheduled sampling \citep{bengio2015scheduled} is a curriculum learning method that likewise aims to bridge this gap in sequence prediction tasks by gradually replacing more and more input tokens with model-generated ones.

While effective in areas such as time series forecasting, scheduled sampling is, in our opinion, inapplicable to language modelling due to the discrete nature of text. Replacing ground truth input sequences by interleavings of ground truth and model-generated tokens frequently results in ungrammatical, factually wrong or otherwise incoherent text, which should be avoided at all cost. Moreover, unlike multi-token prediction, the technique originally developed for recurrent neural networks cannot easily be adapted for parallel training setups like the ones of transformer models.

\subsection{Information-theoretic argument}
\label{app:loss-dec}
We give details on the information-theoretic terms appearing in the decomposition in Section~\ref{sect:mismatch} and derive a relative version that similarly allows to decompose multi-token prediction losses. As in Section~\ref{sect:mismatch}, denote by $X$ the next token and by $Y$ the second-next one, and omit conditioning on the preceding context $C$ for ease of notation. In Section~\ref{sect:mismatch}, we decomposed $H(X) + H(Y)$---the quantity of interest for 2-token prediction models---as follows:
\begin{equation} \label{eq:entr-dec}
  H(X) + H(Y) = H(X \mid Y) + 2 I(X; Y) + H(Y \mid X).
\end{equation}
Let us explain each of the terms. The entropy terms denote the uncertainty contained in the ground-truth random variables $X$ and $Y$. \footnote{In particular, they do not refer to \emph{model} predictions.} The term $H(Y \mid X)$ is a classical next-token entropy for the prefix $(C, X)$. The conditional entropy $H(X \mid Y)$ is a more theoretical entity not modelled by causal models. It describes the uncertainty about $X$ given the prefix $C$ and suffix $Y$, and therefore captures the local variations of $X$ that do not affect the continuation of the text $Y$. The mutual information $I(X;Y)$ on the other hand describes the information about $Y$ contained in $X$ (and vice versa) and therefore captures the variations of $X$ which constrain the continuation of the text.

However, the argument given in Section~\ref{sect:mismatch} relies on the assumption that multi-token prediction losses obey a similar decomposition as the sum of the ground-truth entropies themselves. Let us make this rigorous. Denote by $p(x, y)$ the joint distribution of $X$ and $Y$, by $p(x)$ (short for $p_X(x)$) the marginal distribution of $X$ and by $p(y)$ the one of $Y$. Denote the densities of the model's predictions by $q(x, y)$, $q(x)$ and $q(y)$, respectively, conditional distributions by $p(x \mid y)$ and Kullback-Leibler divergence from $q$ to $p$ by $\KL{p}{q}$ and cross-entropy from $q$ to $p$ by $H(p, q)$.
\begin{definition}
    The \emph{conditional cross-entropy} $H(p_{X \mid Y}, q_{X \mid Y})$ of $X$ conditioned on $Y$ from $q$ to $p$ is defined as the expectation under $y$ of the cross-entropy between the distributions $p_X$ and $q_X$ conditioned on $y$, in formulas:
    \[
    H(p_{X \mid Y}, q_{X \mid Y})
    = \E_{y \sim p_Y} H(p_{X \mid Y=y}, q_{X \mid Y=y})
    = \E_{y \sim p_Y} H(p(\cdot \mid y), q(\cdot \mid y)).
    \]
\end{definition}
\begin{definition}
    The \emph{relative mutual information} $I_{p \| q}(X; Y)$ of $X$ and $Y$ from $q$ relative to $p$ is defined by
    \[
    I_{p \| q}(X; Y)
    = \KL{p}{q_X \otimes q_Y} - \KL{p}{q}.
    \]
\end{definition}
We have $I_{p \| q}(X; Y) = H(p_X, q_X) + H(p_Y, q_Y) - H(p, q)$, $I_{p \| p}(X; Y) = I_p(X; Y)$ reduces to standard mutual information under the distribution $p$ and $I_{p \| q}(X; Y)$ is symmetric in $X$ and $Y$ but can be negative.

We have the following relative version of the decomposition $H(X) = H(X \mid Y) + I(X; Y)$.
\begin{lemma}
    $H(p_X, q_X) = H(p_{X \mid Y}, q_{X \mid Y}) + I_{p \| q}(X; Y).$
\end{lemma}
\begin{proof}
We calculate
\begin{align*}
    H(p_X, q_X)
    &= -\sum_x p(x) \log q(x) \\
    &= -\sum_{x,y} p(x, y) \log q(x) \\
    &= -\sum_{x,y} p(x, y) \log \frac{q(x)q(y)}{p(x, y)}\frac{p(x, y)}{q(x, y)}\frac{q(x, y)}{q(y)}  \\
    &= \KL{p}{q_X \otimes q_Y} - \KL{p}{q} - \sum_{x,y} p(y) p(x \mid y) \log q(x \mid y) \\
    &= I_{p \| q}(X; Y) + \sum_y p(y) H(p_{X\mid y}, q_{Y\mid y}) \\
    &= I_{p \| q}(X; Y) + H(p_{X \mid Y}, q_{X \mid Y}).
\end{align*}
\end{proof}
Symmetrizing, we get the desired relative version of $H(X) + H(Y) = H(X \mid Y) + 2 I(X; Y) + H(Y \mid X)$:
\[
    H(p_X, q_X) + H(p_Y, q_Y)
    = H(p_{X \mid Y}, q_{X \mid Y}) + 2 I_{p \| q}(X; Y) + H(p_{Y \mid X}, q_{Y \mid X}).
\]
Setting $p$ to be the empirical distribution of the training data, the left-hand side describes the cross-entropy loss used to train 2-token prediction models. The right-hand side gives the decomposition into a \emph{local} cross-entropy term, a mutual information term with weight two and a shifted next-token cross-entropy term. We interpret this as follows: by adding the term $H(p_Y, q_Y)$ to the loss, 2-token prediction incentivizes models to precompute features which will become useful for predicting $Y$ in the next step and increases the weight of the relative mutual information term in the loss. What does relative mutual information actually mean? By interpreting Kullback-Leibler divergence $\KL{p}{q}$ as the average number of bits needed in addition to send data from $p$ with a code optimized for $q$ instead of $p$, we see that minimizing
\[
I_{p \| q}(X; Y) = \KL{p}{q_X \otimes q_Y} - \KL{p}{q}
\]
means minimizing the average number of additional bits needed to send data from $p$ with a code optimized for $q$ that treats $X$ and $Y$ as independent compared to one that does not. If this number is small, $q$ managed to exploit the mutual information of $X$ and $Y$ under $p$.

\subsection{Lookahead reinforces choice points}
\label{app:choice}
\begin{figure}[ht!]
    \centering
    \includegraphics[width=0.5\linewidth]{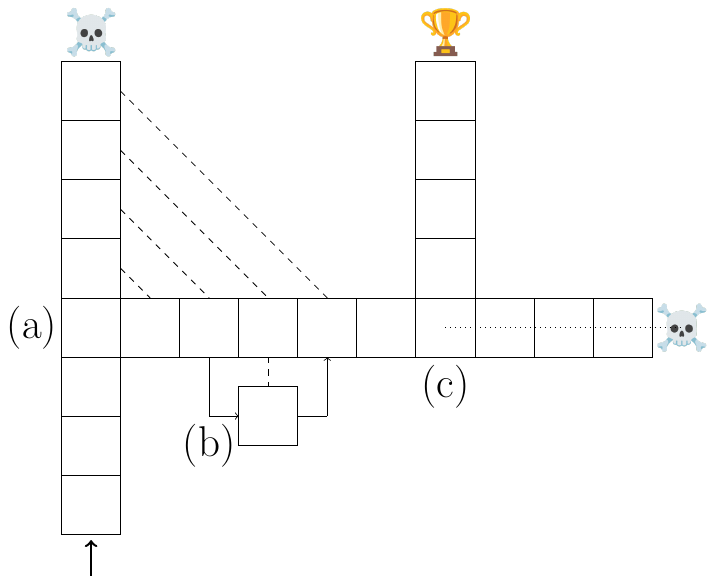}
    \caption{\textbf{Example of a sequential prediction task with derailing.}
    The goal is to go from the arrow to the trophy. Turning around is not allowed. Most transitions are unique, but there are two turns to be taken correctly, the \emph{consequential decisions} (a) and (c). Turn (b) is an \emph{inconsequential decision}: the paths join right after it. Next to transitions (a) and (b), we sketch how a 4-step prediction loss can place more emphasis on consequential transitions than inconsequential ones during teacher-forced training. Next to transition (c), we sketch how a 4-step lookahead can prevent models from taking irreversible suboptimal decisions during autoregressive decoding.
    }
    \label{fig:maze}
\end{figure}

Training with multi-head prediction increases the importance of choice points in the loss in comparison to inconsequential decisions. To make this argument, we present a simplified model of language modelling. Consider a sequential decision task and a model $M$ that is trained in a teacher-forced way on optimal trajectories. We distinguish \emph{choice points} --transitions that lead to different outcomes -- and \emph{inconsequential} decisions which do not (Figure~\ref{fig:maze} (a) and (b)).

More formally, assume that the language model is deployed in a reinforcement learning setting like in \emph{reinforcement learning from human feedback} \cite{ouyang2022training} (states are prompts followed by the partial sequence of tokens $x_{t:1}$ generated so far, actions are single tokens  $x_{t+1}$ to generate, rewards are external $R(x_{t:1})$). The quantity
\[
V_\pi(x_{t:1}) = \mathbb{E}_{x_{t+i} \sim \pi(x_{t+i-1:1}), i \geq 1} \left[ \sum_{i \geq 0} R(x_{t+i:1}) \right]
\]
is the value of the state $x_{t:1}$ following the policy $\pi$, while
\[
\sigma_\pi(x_{t:1}) = \sqrt{\Var_{x_{t+1} \sim \pi(x_{t:1})} \left[ 
V_\pi(x_{t+1:1}) \right] }
\]
quantifies the importance of the decision $x_{t+1}$ on the value thereafter. \emph{Choice points} can formally be viewed as steps $t$ for which $\sigma_\pi(x_{t:1})$ is large, while \emph{inconsequential points} are steps where it is low. Note that for completion models, there is no explicit reward, and our argument is merely meant to illustrate what we mean by \emph{choice points}.

\emph{Derailing} denotes a situation where autoregressive generation of trajectories from $M$ at inference time results in bad outcomes after $M$ made a mistake on a choice point. Even if subsequently, $M$ acts optimally given this choice, the final outcome can be significantly worse than the outcome of the optimal trajectory.

Staying in the teacher-forced setting, we ask: What is the impact of training $M$ with $n$-step prediction instead of next-step prediction on this task? Say $x_t \to x_{t+1}$ is a choice point in an optimal trajectory with the suboptimal choice being $x_t \to \tilde{x}_{t+1}$ (Figure~\ref{fig:maze}~(a)). Assume that the trajectories preceding $x_t$ and succeeding $x_{t+1}$ and $\tilde{x}_{t+1}$ consist of inconsequential transitions, the latter denoted by $\tilde{x}_{t+j} \to \tilde{x}_{t+j+1}$. We will compare the losses of a teacher-forced next-step prediction model and a teacher-forced $n$-step prediction model on the partial trajectory $(x_{t-n+1}, \ldots x_t)$. For the next-step prediction model, the predictions are $(x_{t-n+2}, \ldots, x_t, \tilde{x}_{t+1})$ with a single wrong prediction. The predictions of an $n$-step prediction model at time $t - n + i$, $i = 1, \ldots, n$ are $(x_{t-n+i+1}, \ldots, x_t, \tilde{x}_{t+1}, \ldots, \tilde{x}_{t+i})$ with $i$ wrong predictions. In other words, an $n$-step prediction model receives $1 + \ldots + n = \frac{n(n+1)}{2}$ loss terms pertaining to such a choice point and its consequences, while each inconsequential transition (Figure~\ref{fig:maze}~(b)) is only reinforced $n$ times as often as in a next-step prediction model. In other words, choice points receive on average $\frac{n+1}{2}$ times more importance in the loss of $n$-step prediction models than in next-step prediction models.

As argued in Section~\ref{sect:derailing}, we believe that this model captures important features of training and inference with language models: choice points are semantically important turning points in the generated texts, such as the final answer to a question or a specific line of code, while inconsequential decisions can be a choice among synonyms or of variable names in code.

Apart from this training dynamics point of view, we hypothesize that $n$-step prediction also allows the formation of circuits that specifically spot inconsistencies between predictions for earlier and later steps. For instance, if in an early layer of the model, it can be predicted that a decision $x_t \to \tilde{x}_{t+1}$ leads to suboptimal outcomes $\tilde{x}_{t+n}$ (Figure~\ref{fig:maze}~(c)), subsequent layers can reduce the probability of $x_t \to \tilde{x}_{t+1}$ in the model's next-step prediction. Such behaviors also happen in next-step prediction models given enough capacity, but our experiments in Section~\ref{sect:algorithmic} point to the fact that circuits of this kind are formed more easily in multi-step architectures that enforce the required information $\tilde{x}_{t+n}$ to be available to the model when predicting $\tilde{x}_{t+1}$. We believe that this situation appears frequently in natural language and code modelling, for instance where an initial answer to a question contradicts the results of the \emph{chain of thought} brought forward with the intention to justify it.

In more general terms, this situation arises whenever predicting first $\tilde{x}_{n+i}$ for some $ 1 < i \leq n$ and then $\tilde{x}_{n+1}$ based on $\tilde{x}_{n+i}$ is easier than predicting $\tilde{x}_{n+1}$ directly. We discuss this phenomenon of \emph{factorization orders} in the next section and present a specific instance of it that frequently appears in modelling natural language.

\subsection{Factorization orders}
Causal language modelling factorizes probabilities over text sequences $x_t \cdots x_1$ classically as 
\[ P(x_t \cdots x_1) = \prod_{i=1}^t P(x_i \,|\, x_{i-1} \cdots x_1). \]
While moving forward in time is certainly the most natural choice of factorization order, there exist cases where it is suboptimal. In inflectional languages, for instance, agreement between related sentence parts is a frequent pattern with one word directing the grammatical forms of others. Consider the German sentence
\begin{quote}
Wie konnten auch Worte meiner durstenden Seele genügen?\footnote{roughly: \textit{How could words be enough for my thirsty soul?}}
\par\noindent\hfill\textit{Friedrich Hölderlin, Fragment von Hyperion (1793)}
\end{quote}
where "genügen" requires a dative case object and then "Seele" requires the possessive pronoun "mein" to be in female singular dative form "meiner" and the participle "durstend" to be in female singular dative form in weak declination "durstenden" because it follows "meiner". In other words, the factorization order
\begin{quote}
Wie konnten auch Worte $\rightarrow$ genügen $\rightarrow$ Seele $\rightarrow$ meiner $\rightarrow$ durstenden?
\end{quote}
is arguably an easier one for constructing the above sentence. Humans as well as language models therefore have to perform this factorization (which deviates from the causal order in which predictions take place!) within their latent activations, and a $4$-token prediction loss makes this easier as it explicitly encourages models to have all information about the successive 4 tokens in its latent representations.

\newpage
\section{Training hyperparameters}
\begin{table}[H]
    \centering
        \caption{\textbf{Overview of all training hyperparameters used.} We schedule all learning rates with a linear warmup and cosine decay \citep{loshchilov2017sgdr} to a fraction of the peak learning rate which is depicted in the last column (``decay ratio''). All experiments use the Adam \citep{kingma2014adam} optimizer with $\beta_1 = 0.9$, $\beta_2 = 0.95$ and decoupled $L_2$ weight decay \citep{loshchilov2019decoupled} coefficient $0.1$. We clip gradients to a maximal Euclidean norm of $1.0$ in all experiments except CodeContests finetunings, where we use $0.1$ instead. Summarization finetunings correspond to three epochs on all datasets except BigPatent (1 epoch). Byte-level models use the architecture with replicated unembeddings from Appendix~\ref{app:architecture}.}
    \label{tab:hyperparams}
    \resizebox{0.98\textwidth}{!}{
    \begin{tabular}{lrrrrrrr}
        \toprule
        Model & Batch size ($2^{20}$) & Steps & Tokens (B) & Warmup steps & Peak LR & Context length & Decay ratio \\
        \midrule
        \multicolumn{2}{l}{Model scaling (Section~\ref{sect:model-scaling})} \\
        \cmidrule(r{3em}){1-2}
        0.3B & 8 & 10,850 & 91.0 & 1000 & $3 \times 10^{-4}$ & 4096 & 0.03 \\
        0.6B & 8 & 10,850 & 91.0 & 1000 & $3 \times 10^{-4}$ & 4096 & 0.03 \\
        1.3B & 8 & 10,850 & 91.0 & 1000 & $3 \times 10^{-4}$ & 4096 & 0.03 \\
        3B & 8 & 10,850 & 91.0 & 1000 & $3 \times 10^{-4}$ & 4096 & 0.03 \\
        7B & 8 & 25,000 & 209.7 & 2000 & $3 \times 10^{-4}$ & 4096 & 0.03 \\
        13B & 8 & 25,000 & 209.7 & 1000 & $3 \times 10^{-4}$ & 4096 & 0.03 \\
        \midrule
        \multicolumn{2}{l}{Code models (Section~\ref{sect:results})} \\
        \cmidrule(r{3em}){1-2}
        7B 200B & 8 & 25,000 & 209.7 & 2000 & $3 \times 10^{-4}$ & 4096 & 0.03 \\
        7B 500B & 7 & 68,570 & 503.3 & 2000 & $3 \times 10^{-4}$ & 4096 & 0.03 \\
        7B 1T & 7 & 136,240 & 1000.0 & 2000 & $3 \times 10^{-4}$ & 4096 & 0.03 \\
        \midrule
        \multicolumn{2}{l}{Byte-level models (Section~\ref{sect:byte-level})} \\
        \cmidrule(r{2em}){1-2}
        7B 314GB & 12 & 25,000 & 314.6 & 2000 & $3 \times 10^{-4}$ & 8192 & 0.03 \\
        \midrule
        \multicolumn{2}{l}{Language models (Section~\ref{sect:nlp})} \\
        \cmidrule(r{2em}){1-2}
        7B 200B & 8 & 25,000 & 209.7 & 2000 & $3 \times 10^{-4}$ & 4096 & 0.10 \\
        7B 500B & 8 & 60,000 & 503.3 & 2000 & $3 \times 10^{-4}$ & 4096 & 0.10 \\
        \midrule
        \multicolumn{2}{l}{Induction task (Section~\ref{sect:induction})} \\
        \cmidrule(r{3em}){1-2}
        1M -- 1B& 0.25 & 100,000 & 26.2 & 2000 & $10^{-4}$ & 2048 & 0.03 \\
        1M -- 1B (Appendix~\ref{app:induction}) & 0.5 & 50000 & 26.2 & 2000 & $10^{-4}$ & 2048 & 0.03 \\
        \midrule
        \multicolumn{2}{l}{Arithmetic task (Section~\ref{sect:algorithmic})} \\
        \cmidrule(r{3em}){1-2}
        30M & 0.25 & 100,000 & 26.2 & 2000 & $10^{-4}$ & 1024 & 0.03 \\
        100M & 0.25 & 100,000 & 26.2 & 2000 & $10^{-4}$ & 2048 & 0.03 \\
        \midrule
        \multicolumn{2}{l}{Summarization (Section~\ref{sect:nlp})} \\
        \cmidrule(r{3em}){1-2}
        BigPatent & 0.125 & 76,680 & 10.1 & 100 & $3 \times 10^{-5}$ & 4096 & 0.03 \\
        CNN/Dailymail & 0.125 & 7,140 & 0.9 & 100 & $3 \times 10^{-5}$ & 4096 & 0.03 \\
        Multi-News & 0.125 & 3,330 & 0.4 & 100 & $3 \times 10^{-5}$ & 4096 & 0.03 \\
        OrangeSum & 0.125 & 360 & 0.0 & 100 & $3 \times 10^{-5}$ & 4096 & 0.03 \\
        pn-summary & 0.125 & 3,450 & 0.5 & 100 & $3 \times 10^{-5}$ & 4096 & 0.03 \\
        SAMSum & 0.125 & 60 & 0.0 & 100 & $3 \times 10^{-5}$ & 4096 & 0.03 \\
        ThaiSum & 0.125 & 23,640 & 3.1 & 100 & $3 \times 10^{-5}$ & 4096 & 0.03 \\
        WikiSummary & 0.125 & 2,550 & 0.3 & 100 & $3 \times 10^{-5}$ & 4096 & 0.03 \\
        XSum & 0.125 & 2,760 & 0.4 & 100 & $3 \times 10^{-5}$ & 4096 & 0.03 \\
        \midrule
        \multicolumn{2}{l}{CodeContests (Section~\ref{sect:finetuning})} \\
        \cmidrule(r{3em}){1-2}
        7B & 0.25 & 13,000 & 3.6 & 400 & $5 \times 10^{-5}$ & 4096 & 0.004 \\
        \bottomrule
    \end{tabular}}
\end{table}

\begin{table}[!ht]
    \centering
        \caption{\textbf{Overview of model architectures used for scaling analyses.}}
    \label{tab:models}
    \begin{tabular}{lrrr}
        \toprule
        Name & Dimension & Layers & Heads \\
        \midrule
        1M & 128 & 5 & 4 \\
        3M & 256 & 4 & 8 \\
        10M & 384 & 6 & 8 \\
        30M & 512 & 10 & 8 \\
        100M & 768 & 14 & 12 \\
        300M & 1024 & 25 & 16 \\
        1B & 1536 & 36 & 24 \\
        \midrule
        0.3B & 1024 & 18 & 16 \\
        0.6B & 1280 & 27 & 20 \\
        1.3B & 2048 & 24 & 16 \\
        3B & 2560 & 36 & 20 \\
        6.7B (``7B'') & 4096 & 32 & 32 \\
        13B & 5120 & 40 & 40 \\
        \bottomrule
    \end{tabular}
\end{table}

\end{document}